\documentclass{article}

\usepackage{amsthm}
\usepackage{amsmath}
\usepackage{amssymb}
\usepackage{microtype}
\usepackage{graphicx}
\usepackage{subfigure}
\usepackage{booktabs}
\usepackage{hyperref}

\usepackage[accepted]{icml2019}
\icmltitlerunning{Target Tracking for Contextual Bandits: Application to Demand Side Management}

\begin{document}

\twocolumn[
\icmltitle{Target Tracking for Contextual Bandits: \\
          Application to Demand Side Management}

\icmlsetsymbol{equal}{*}

\begin{icmlauthorlist}
\icmlauthor{Margaux Br{\'e}g{\`e}re}{edf,psud,inria}
\icmlauthor{Pierre Gaillard}{inria}
\icmlauthor{Yannig Goude}{edf,psud}
\icmlauthor{Gilles Stoltz}{psud}
\end{icmlauthorlist}

\icmlaffiliation{edf}{EDF R{\&}D, Palaiseau, France}
\icmlaffiliation{inria}{INRIA - D{\'e}partement d'Informatique de l'{\'E}cole Normale Sup{\'e}rieure, PSL Research University, Paris, France}
\icmlaffiliation{psud}{Laboratoire de math{\'e}matiques d'Orsay, Universit{\'e} Paris-Sud, CNRS, Universit{\'e} Paris-Saclay, Orsay, France}

\icmlcorrespondingauthor{Margaux Br{\'e}g{\`e}re}{margaux.bregere@edf.fr}

\vskip 0.3in
]

\printAffiliationsAndNotice{}

\begin{abstract}
We propose a contextual-bandit approach for demand side management by offering price incentives. More
precisely, a target mean consumption is set at each round and the mean consumption is modeled as a complex
function of the distribution of prices sent and of some contextual variables such as the temperature, weather,
and so on. The performance of our strategies is measured in quadratic losses through a regret criterion. We
offer $T^{2/3}$ upper bounds on this regret (up to poly-logarithmic terms)---and even faster rates under stronger assumptions---for strategies inspired by standard
strategies for contextual bandits (like LinUCB, see \citealp{li2010contextual}). Simulations on a real data set
gathered by UK Power Networks, in which price incentives were offered, show that our strategies are effective and
may indeed manage demand response by suitably picking the price levels.
\end{abstract}

\newcommand{\cX}{\mathcal{X}}
\newcommand{\cP}{\mathcal{P}}
\newcommand{\cF}{\mathcal{F}}
\newcommand{\cH}{\mathcal{H}}
\newcommand{\E}{\mathbb{E}}
\renewcommand{\P}{\mathbb{P}}
\newcommand{\R}{\mathbb{R}}
\renewcommand{\leq}{\leqslant}
\renewcommand{\geq}{\geqslant}
\newcommand{\transp}{\mbox{\tiny \textup{T}}}
\renewcommand{\epsilon}{\varepsilon}
\newcommand{\Var}{\mathop{\mathrm{Var}}}
\newcommand{\hth}{\widehat{\theta}}
\newcommand{\Id}{\mathrm{I}_d}
\newcommand{\defeq}{\stackrel{\mbox{\tiny \rm def}}{=}}
\renewcommand{\hat}{\widehat}
\newcommand{\norm}[1]{\Arrowvert #1 \Arrowvert}
\newcommand{\bnorm}[1]{\bigl\Arrowvert #1 \bigr\Arrowvert}
\newcommand{\Bnorm}[1]{\Bigl\Arrowvert #1 \Bigr\Arrowvert}
\newcommand{\e}{\mathrm{e}}
\newcommand{\oL}{\overline{L}}
\newcommand{\oR}{\overline{R}}
\newcommand{\oS}{\overline{S}}
\newcommand{\oB}{\overline{B}}
\renewcommand{\d}{\,\mathrm{d}}
\newcommand{\hl}{\widehat{\ell}}
\newcommand{\tl}{\widetilde{\ell}}
\newcommand{\argmin}{\mathop{\mathrm{arg\,min}}}
\newcommand{\cM}{\mathcal{M}}
\newcommand{\ind}[1]{\mathbf{1}_{\{ #1 \}}}

\newcommand{\note}[1]{[{\color{red} #1}]}
\newcommand{\Tr}{\mathrm{Tr}}
\renewcommand{\log}{\ln}
\def\mystrut(#1,#2){\vrule height #1pt depth #2pt width 0pt}

\newtheorem{lemma}{Lemma}
\newtheorem{proposition}{Proposition}
\newtheorem{theorem}{Theorem}

\newcounter{examplecounter}
\newenvironment{example}{\refstepcounter{examplecounter} \textbf{Example \arabic{examplecounter}.}~~}{\hfill \qed}
\newcounter{modelcounter}
\newenvironment{model}[1]{\refstepcounter{modelcounter} \textbf{Model \arabic{modelcounter}: {#1}.}~~}{}
\newcounter{asscounter}

\newenvironment{comm}{\textbf{Comment:}~~}{}

\newcounter{protocol}
\makeatletter
\newenvironment{protocol}[1][htb]{
  \let\c@algorithm\c@protocol
  \renewcommand{\ALG@name}{Protocol}
  \begin{algorithm}[#1]
  }{\end{algorithm}
}
\makeatother

\vspace*{-.5cm}
\section{Introduction}

Electricity management is classically performed by anticipating demand and adjusting accordingly production.
The development of smart grids, and in particular the installation of smart meters (see \citealp{Yan13,mallet14}), come with new opportunities:
getting new sources of information, offering new services. For example,
demand-side management (also called demand-side response; see \citealp{4275494,SIANO2014461} for an overview)
consists of reducing or increasing consumption of electricity users when needed, typically reducing at peak times and encouraging consumption of off-peak times.
This is good to adjust to intermittency of renewable energies and is made possible by the
development of energy storage devices such as batteries or even electric vehicles (see \citealp{Fischer15, Kikusato18}); the storages at hand can take place
at a convenient moment for the electricity provider.
We will consider such a demand-side management system, based on price incentives sent to users via their smart meters.
We propose here to adapt contextual bandit algorithms to that end, which are already used in online advertising.
Other such systems were based on different heuristics (\citealp{shareef2018review,wang2015load}).

The structure of our contribution is to first provide a modeling
of this management system, in Section~\ref{sec:mod}. It relies on
making the mean consumption as close as possible to a moving target
by sequentially picking price allocations.
The literature discussion of the main ingredient of our algorithms, contextual bandit theory,
is postponed till Section~\ref{sec:lit}.
Then, our main results are stated and discussed in Section~\ref{sec:mainresult}:
we control our cumulative loss through a $T^{2/3}$ regret bound
with respect to the best constant price allocation.
A refinement as far as convergence rates are concerned is offered in
Section~\ref{sec:fast}.
A section with simulations based on a real data set concludes the paper:
Section~\ref{sec:simus}.
For the sake of length, most of the proofs are provided in the supplementary material.
\vspace*{-5pt}

\paragraph{Notation.} Without further indications, $\norm{x}$ denotes the Euclidean norm of a vector $x$.
For the other norms, there will be a subscript: e.g., the supremum norm of $x$ is
is denoted by $\norm{x}_\infty$.

\section{Setting and Model}
\label{sec:mod}

Our setting consists of a modeling of electricity consumption
and of an aim---tracking a target consumption. Both rely
on price levels sent out to the customers.

\subsection{Modeling of the Electricity Consumption}
\label{sec:modeling}

We consider a large population of customers of some electricity provider and assume it homogeneous, which is not an uncommon
assumption, see \citet{mei2017nonnegative}.
The consumption of each customer at each instance $t$ depends, among others,
on some exogenous factors (temperature,
wind, season, day of the week, etc.), which will form a context vector $x_t \in \cX$,
where $\cX$ is some parametric space. The electricity provider aims to manage demand response: it sets a target mean consumption $c_t$ for each time instance.
To achieve it, it changes electricity prices accordingly (by making it more
expensive to reduce consumption or less expensive to encourage customers to consume more now
rather than in some hours). We assume that $K \geq 2$ price levels (tariffs) are available.
The individual consumption of a given customer getting tariff $j \in \{1,\ldots,K\}$
is assumed to be of the form $\varphi(x_t,j)+~\mbox{white noise}$,
where the white noise models the variability due to the customers, and where $\varphi$
is some function associating with a context $x_t$ and a tariff $j$
an expected consumption $\varphi(x_t,j)$. Details on and examples of
$\varphi$ are provided below. At instance $t$, the electricity provider sends
tariff $j$ to a share $p_{t,j}$ of the customers; we denote by $p_t$ the convex vector
$(p_{t,1},\ldots,p_{t,K})$. As the population is rather homogeneous, it is unimportant
to know to which specific customer a given signal was sent; only the global proportions
$p_{t,j}$ matter. The mean consumption observed equals \vspace*{-10pt}
\[
	\textstyle{Y_{t,p_t} = \sum_{j=1}^K p_{t,j} \, \varphi(x_t,j) + \mbox{noise}\,.}
\]
The noise term is to be further discussed below; we first focus on the
$\varphi$ function by means of examples.

\begin{example}
The simplest approach consists in considering
a linear model per price level, i.e., parameters
$\theta_1, \ldots, \theta_K \in \R^{\dim(\cX)}$
with $\varphi(x_t,j) = \theta_j^{\transp} x_t$.
We denote $\theta = (\theta_j)_{1 \leq j \leq K}$
the vector formed by aggregating all vectors~$\theta_j$.

This approach can be generalized by replacing $x_t$ by
a vector-valued function $b(x_t)$.
This corresponds to the case where it is assumed that
the $\varphi(\,\cdot\,,j)$ belong to some set $\cH$ of
functions $h : \cX \to \R$, with
a basis composed of $b_1,\ldots,b_q$. Then,
$b = (b_1,\ldots,b_q)$. For instance, $\cH$ can be given by histograms
on a given grid of $\cX$.
\end{example}

\begin{example}
\label{ex:2}
Generalized additive models (\citealp{wood2006generalized})
form a powerful and efficient semi-parametric approach to model electricity consumption
(see, among others, \citealp{goude2014local,gaillard2016additive}).
It models the load as a sum of independent exogenous variable  effects.
In our simulations, see~\eqref{eq:gam}, we will consider a mean expected consumption
of the form $\varphi(x_t,j) = \varphi(x_t,0) + \xi_j$,
that is, the tariff will have a linear impact on the mean consumption, independently of the
contexts.
The baseline mean consumption $\varphi(x_t,0)$ will be modeled as a sum of simple $\R \to \R$
functions, each taking as input a single component of the context vector: \vspace{-20pt}

\[
	\smash{\textstyle{\varphi(x_t,0) = \sum_{i=1}^Q f^{(i)} ( x_{t,h(i)} )}\,,} \mystrut(8,0) \vspace{-3pt}
\]
where $Q \geq 1$ and where each $h(i) \in \bigl\{ 1,\ldots,\dim(\cX) \bigr\}$.
Some components $h(i)$ may be used several times.
When the considered component \smash{$x_{t,h(i)}$} takes continuous values,
these functions $f^{(i)}$ are so-called cubic splines:
$\mathcal{C}^2$--smooth functions made up of sections of cubic polynomials joined together at points of a grid
(the knots). Choosing the number $q_i$ of knots (points at which the sections join) and their locations is sufficient to determine
(in closed form) a linear basis $\bigl( b^{(i)}_1,\ldots,b^{(i)}_{q_i})$ of size $q_i$, see~\citet{wood2006generalized} for
details. The function $f^{(i)}$ can then
be represented on this basis by a vector of length $q_i$, denoted by $\theta^{(i)}$:
\[
	\smash{ \textstyle{f^{(i)} = \sum_{j=1}^{q_i} \theta^{(i)}_j b^{(i)}_j\,.}} \mystrut(8,0)
\]
When the considered component \smash{$x_{t,h(i)}$} takes finitely many values,
we write $f^{(i)}$ as a sum of indicator functions: \vspace*{-3pt}
\[
	\smash{\textstyle{f^{(i)} = \sum_{j=1}^{q_i} \theta^{(i)}_j \ind{v^{(i)}_j}}} \mystrut(8,0) \,,
\]
where the $v^{(i)}_j$ are the
$q_i$ modalities for the component $h(i)$.
All in all, $\varphi(x_t,j)$ can be represented by
a vector of dimension $K + q_1 + \ldots + q_Q$
obtained by aggregating the $\xi_j$ and
the vectors $\theta^{(i)}$ into a single vector.
\end{example}

Both examples above show that it is reasonable to assume that
there exists some unknown $\theta \in \R^d$ and some known
transfer function $\phi$ such that $\varphi(x_t,j) = \phi(x_t,j)^{\transp} \theta$.
By linearly extending $\phi$ in its second component, we get \vspace*{-5pt}
\[
Y_{t,p_t} = \phi(x_t,p_t)^{\transp} \theta + \mbox{noise}\,. \vspace*{-3pt}
\]
We will actually not use in the sequel that $\phi(x,p)$ is linear in $p$:
the dependency of $\phi(x,p)$ in $p$ could be arbitrary.

We now move on to the noise term. We first recall that we assumed that our population
is rather homogeneous, which is a natural feature as soon as it is large enough.
Therefore, we may assume that the variabilities within the group of customers
getting the same tariff $j$ can be combined into a single random variable $\varepsilon_{t,j}$.
We denote by $\epsilon_t$ the vector $(\varepsilon_{t,1}, \ldots, \varepsilon_{t,K})$.
All in all, we will mainly consider the following model.

\begin{model}{tariff-dependent noise}
When the electricity provider picks the convex vector $p$,
the mean consumption obtained at time instance $t$ equals \vspace*{-3pt}
\[
	Y_{t,p} = \phi(x_t,p)^{\transp} \theta + p^{\transp} \epsilon_t\,. \vspace*{-3pt}
\]
The noise vectors $\epsilon_1,\epsilon_2,\ldots$ are $\rho$--sub-Gaussian\footnote{\label{fn:sG} A
$d$--dimensional random vector $\epsilon$ is $\rho$--sub-Gaussian, where $\rho > 0$, if
for all $\nu \in \R^d$, one has $\E \bigl[ \e^{\nu^{\transp} \epsilon} \bigr] \leq \e^{\rho^2 \norm{\nu}^2 /2}$.}
i.i.d.\ random variables with $\E[\epsilon_1] = (0,\ldots,0)^{\transp}$.
We denote by $\Gamma = \Var(\epsilon_1)$ their covariance
matrix.
No assumption is made on $\Gamma$
in the model above (real data confirms that $\Gamma$ typically has no special form, see~\ref{subsec:simulator}).
However, when it is proportional to the $K \times K$ matrix $[1]$,
the noises associated with each group can be combined into a global noise,
leading to the following model. It is less realistic in practice, but we discuss
it because regret bounds may be improved in the presence of a global noise.
\end{model}

\begin{model}{global noise}
When the electricity provider picks the convex vector $p$,
the mean consumption obtained at time instance $t$ equals \vspace*{-8pt}
\[
\hspace{3.2cm} Y_{t,p} = \phi(x_t,p)^{\transp} \theta + e_t\,. \vspace*{-3pt}
\]
The scalar noises $e_1,e_2,\ldots$ are $\rho$--sub-Gaussian i.i.d.\ random variables,
with $\E[e_1] = 0$. We denote by
$\sigma^2 = \Var(e_1)$ the variance of
the random noises $e_t$.
\end{model}

\subsection{Tracking a Target Consumption}
\label{sec:tracking}

We now move on to the aim of the electricity provider.
At each time instance $t$, it picks an allocation of
price levels $p_t$ and wants the observed mean consumption $Y_{t,p_t}$
to be as close as possible to some target mean consumption $c_t$.
This target is set in advance by another branch of the provider and $p_t$ is to be
picked based on this target: our algorithms will explain how to pick $p_t$
given $c_t$ but will not discuss the choice of the latter.
In this article we will measure the discrepancy between the
observed $Y_{t,p_t}$ and the target $c_t$ via a quadratic loss:
$(Y_{t,p_t} - c_t)^2$.
We may set some restrictions on the convex combinations $p$ that can be picked:
we denote by $\cP$ the set of legible allocations of price levels. This models
some operational or marketing constraints that the electricity provider may encounter.
We will see that whether $\cP$ is a strict subset of all convex vectors or whether it is given
by the set of all convex vectors plays no role in our theoretical analysis.

As explained in Section~\ref{sec:regret},
we will follow a standard path in online learning theory:
to minimize the cumulative loss suffered we will minimize some regret.

\subsection{Summary: Online Protocol}

After picking an allocation of price levels $p_t$,
the electricity provider only observes $Y_{t,p_t}$:
it thus faces a bandit monitoring.
Because of the contexts $x_t$, the problem considered falls
under the umbrella of contextual bandits.
No stochastic assumptions are made on the sequences $x_t$ and $c_t$:
the contexts $x_t$ and $c_t$ will be considered as picked by the environment.
Finally, mean consumptions are assumed to be bounded between $0$ and $C$,
where $C$ is some known maximal value.
The online protocol described in Sections~\ref{sec:modeling}
and~\ref{sec:tracking}
is stated in Protocol~\ref{prot:1}.
We see that the choices $x_t$, $c_t$ and $p_t$ need to be
$\cF_{t-1}$--measurable, where  \\[3pt]
\hspace*{50pt} $\cF_{t-1} \defeq \sigma(\epsilon_1,\ldots,\epsilon_{t-1})\,.$

\setlength{\textfloatsep}{5pt}% Remove \textfloatsep

\begin{protocol}[tb]
\caption{\label{prot:1} Target Tracking for Contextual Bandits}
\begin{algorithmic}
\STATE \textbf{Input}
\STATE \quad Parametric context set $\cX$
\STATE \quad Set of legible convex weights $\cP$
\STATE \quad Bound on mean consumptions $C$
\STATE \quad Transfer function $\phi : \cX \times \cP \to \R^d$
\STATE \textbf{Unknown parameters}
\STATE \quad Transfer parameter $\theta \in \R^d$
\STATE \quad Covariance matrix $\Gamma$ of size $K \times K$ \hfill (Model~1)
\STATE \quad Variance $\sigma^2$ \hfill (Model~2)\\[3pt]

\FOR{$t=1,2,\ldots$}
\STATE Observe a context $x_t \in \cX$ and a target $c_t \in (0,C)$
\STATE Choose an allocation of price levels $p_t \in \cP$
\STATE Observe a resulting mean consumption
\begin{align*}
Y_{t,p_t} & = \phi(x_t,p_t)^{\transp} \theta + p_t^{\transp} \epsilon_t \qquad & \mbox{(Model~1)} \\
Y_{t,p_t} & = \phi(x_t,p_t)^{\transp} \theta + e_t & \mbox{(Model~2)}
\end{align*}
\STATE Suffer a loss $(Y_{t,p_t} - c_t)^2$
\ENDFOR \\[3pt]

\STATE  \textbf{Aim} \vspace{-.2cm}
\STATE \quad Minimize the cumulative loss \hfill $\displaystyle{L_T = \sum_{t=1}^T (Y_{t,p_t} - c_t)^2}$
\end{algorithmic}
\end{protocol}

\subsection{Literature Discussion: Contextual Bandits}
\label{sec:lit}

In many bandit problems the learner has access to additional information at the beginning of each round. Several settings for this
side information may be considered. The adversarial case was introduced in~\citet[Section~7, algorithm Exp4]{auer2002nonstochastic}:
and subsequent improvements were suggested in~\citet{beygelzimer2011contextual} and \citet{mcmahan2009tighter}.
The case of i.i.d.\ contexts with rewards depending on contexts through an unknown parametric model was
introduced by~\citet{wang2005bandit} and generalized to the non-i.i.d.\ setting in~\citet{wang2005arbitrary},
then to the multivariate and nonparametric case in~\citet{perchet2013multi}. Hybrid versions (adversarial contexts
but stochastic dependencies of the rewards on the contexts, usually in a linear fashion) are the most popular ones.
They were introduced by~\citet{abe1999associative} and further studied
in~\citet{auer2002using}. A key technical ingredient to deal with them is confidence ellipsoids
on the linear parameter; see~\citet{dani2008stochastic}, \citet{rusmevichientong2010linearly} and \citet{abbasi2011improved}.
The celebrated UCB algorithm of~\citet{lai1985asymptotically} was generalized in this hybrid setting as the LinUCB
algorithm, by \citet{li2010contextual} and \citet{chu2011contextual}. Later, \citet{NIPS2010_4166} extended it to a setting with
generalized additive models and \citet{valko2013finite} proposed a kernelized version of UCB.
Other approaches, not relying on confidence ellipsoids, consider sampling strategies (see \citealp{gopalan2014thompson}) and are currently
extended to bandit problems with complicated dependency in contextual variables \citep{mannor2018jds}.
Our model falls under the umbrella of hybrid versions considering stochastic linear bandit problems given a context.
The main difference of our setting lies in how we measure performance: not directly with the rewards or their analogous
quantities $Y_{t,p_t}$ in our setting, but through how far away they are from the targets $c_t$.

\section{Main Result, with Model~1}
\label{sec:mainresult}

This section considers Model~1.
We take inspiration from LinUCB (\citealp{li2010contextual,chu2011contextual}):
given the form of the observed mean consumption, the key is to
estimate the parameter $\theta$. Denoting by $\Id$ the
$d \times d$ identity matrix and picking $\lambda > 0$, we classically do so according to \vspace{-2pt}
\begin{equation}
\label{eq:hth}
 \hth_t \defeq \smash{V_t^{-1} \sum_{s=1}^t Y_{s,p_s} \phi(x_s,p_s)} \mystrut(10,10) \vspace{-1pt}
\end{equation}
where $\smash{V_t \defeq \lambda \Id + \sum_{s=1}^t \phi(x_s,p_s) \phi(x_s,p_s)^{\transp}}$.

A straightforward adaptation of earlier results (see Theorem~2 of \citealp{abbasi2011improved} or Theorem~20.2 in the monograph by \citealp{lattimore2018bandit}) yields the following
deviation inequality; details are provided in the supplementary material (Appendix~\ref{app:conc-theta}).

\begin{lemma}
\label{lm:conc-theta}
No matter how the provider picks the $p_t$, we have, for all $t \geq 1$
and all $\delta \in (0,1)$, \vspace*{-5pt}
\begin{multline*}
\sqrt{\bigl( \hth_t - \theta \bigr)^{\transp} {V_t} \bigl( \hth_t - \theta \bigr)}
\defeq
\bigl\Arrowvert V_t^{1/2} \bigl( \hth_t - \theta \bigr) \bigr\Arrowvert \\ \leq
\sqrt{\lambda} \Arrowvert \theta \Arrowvert + \rho \sqrt{2 \ln \frac{1}{\delta} + d \ln \frac{1}{\lambda} + \ln \det(V_t)}\,,
\end{multline*}
with probability at least~$1-\delta$.
\end{lemma}

Actually, the result above could be improved into an anytime
result (``with probability $1-\delta$, for all $t \geq 1$, ...'')
with no effort, by applying a stopping argument (or, alternatively,
Doob's inequality for super-martingales), as~\citet{abbasi2011improved} did.
This would slightly improve the regret bounds below by logarithmic factors.
\vspace{-4pt}

\subsection{Regret as a Proxy for Minimizing Losses}
\label{sec:regret}

We are interested in the cumulative sum of the losses, but
under suitable assumptions (e.g., bounded noise)
the latter is close to the sum of the conditionally expected losses
(e.g., through the Hoeffding--Azuma inequality). Typical statements are of the form:
for all strategies of the provider and of the environment, with probability at least $1-\delta$,
\vspace*{-7pt}
\begin{align*}
L_T & = \ \smash{\textstyle{\sum_{t=1}^T (Y_{t,p_t} - c_t)^2}} \mystrut(10,0) \\
& \leq \ \smash{ \textstyle{\sum_{t=1}^T \E\bigl[ (Y_{t,p_t} - c_t)^2 \,\big|\, \cF_{t-1} \bigr]  + O \big( \sqrt{T \ln (1/\delta)} \big).}} \mystrut(10,0)
\end{align*}
All regret bounds in the sequel will involve
the sum of conditionally expected losses $\oL_T$ above
but up to adding a deviation
term to all these regret bounds, we get from them a bound on the true
cumulative loss $L_T$.
Now, the choices $x_t$, $c_t$ and $p_t$ are
$\cF_{t-1}$--measurable, where $\cF_{t-1} = \sigma(\epsilon_1,\ldots,\epsilon_{t-1})$.
Therefore, under Model~1,
\begin{align}
\E  \bigl[(Y_{t,p_t} & - c_t)^2  \,\big|\,  \cF_{t-1} \bigr]  \nonumber \\
& = \E\Big[\big(\phi(x_t,p_t)^{\transp} \theta + p_t^{\transp} \epsilon_t - c_t \big)^2 \,\Big|\, \cF_{t-1} \Big] \nonumber \\
& = \bigl( \phi(x_t,p_t)^{\transp} \theta - c_t \bigr)^2 + \E\big[ (p_t^{\transp} \epsilon_t)^2 \,\big|\, \cF_{t-1} \big] \nonumber \\
& \qquad \quad  + \E \Big[ 2 \bigl( \phi(x_t,p_t)^{\transp} \theta - c_t \bigr) p_t^{\transp} \epsilon_t \,\Big|\, \cF_{t-1} \Big] \nonumber \\
& = \bigl( \phi(x_t,p_t)^{\transp} \theta - c_t \bigr)^2 + p_t^{\transp} \Gamma p_t\,,
\label{eq:ltp} 
\vspace{-3pt} 
\end{align}
that is, after summing, \vspace{-4pt}
\[
~~~~~~~~~~~~~~~~~~~~~~~\textstyle{\oL_T = \sum_{t=1}^T \bigl( \phi(x_t,p_t)^{\transp} \theta - c_t \bigr)^2 + p_t^{\transp} \Gamma p_t}\,.
\]
\newpage

We therefore introduce the (conditional) regret \vspace*{-3pt}
\begin{align*}
\oR_T & = \sum_{t=1}^T \bigl( \phi(x_t,p_t)^{\transp} \theta - c_t \bigr)^2 + p_t^{\transp} \Gamma p_t\\
& \qquad - \smash{\sum_{t=1}^T \min_{p \in \cP} \Bigl\{ \bigl( \phi(x_t,p)^{\transp} \theta - c_t \bigr)^2 + p^{\transp} \Gamma p \Bigr\}\,. \mystrut(20,8)}
\end{align*}
This will be the quantity of interest in the sequel. \vspace{-5pt}

\subsection{Optimistic Algorithm: All but the Estimation of $\Gamma$}
\label{sec:optiregret}

We assume that in the first $n$ rounds
an estimator $\hat{\Gamma}_n$ of the covariance matrix $\Gamma$
was obtained; details are provided in the next subsection.
We explain here how the algorithm plays for rounds $t \geq n+1$.
We assumed that the transfer function $\phi$ and the bound $C > 0$ on the target mean consumptions were known.
We use the notation
$
[x]_C = \min\bigl\{\max\{x,0\},\,C \bigr\}
$
for the clipped part of a real number $x$ (clipping between $0$ and $C$).
We then estimate the instantaneous losses~\eqref{eq:ltp}
\[
\ell_{t,p} \defeq \E \bigl[(Y_{t,p} - c_t)^2 \,\big|\,  \cF_{t-1} \bigr]
= \bigl( \phi(x_t,p)^{\transp} \theta - c_t \bigr)^2 + p^{\transp} \Gamma p
\]
associated with each choice $p \in \cP$ by: \vspace{-5pt}
\[
\hl_{t,p} = \Bigl( \bigl[ \phi(x_t,p)^{\transp} \hth_{t-1} \bigr]_C - c_t \Bigr)^2 + p^{\transp} \hat{\Gamma}_n p\,.
\]
We also denote by $\alpha_{t,p}$ deviation bounds, to be set by the analysis.
The optimistic algorithm picks, for $t \geq n+1$: \vspace{-5pt}
\begin{equation}
\label{eq:optialgo}
p_t \in \argmin_{p \in \cP} \bigl\{ \hl_{t,p} - \alpha_{t,p} \bigr\}\,.
\end{equation}
\vspace{-10pt}

\begin{comm}
In linear contextual bandits, rewards are linear in $\theta$ and to maximize global gain, LinUCB \cite{li2010contextual}
picks a vector $p$ which maximizes a sum of the form $\phi(x_t,p)^{\transp} \hat{\theta}_{t-1} + \tilde{\alpha}_{t,p}$.
Here, as we want to track the target, we slightly change this expression by substituting the target $c_t$ and
taking a quadratic loss. But the spirit is similar.
\end{comm}

\subsection{Optimistic Algorithm: Estimation of $\Gamma$}
\label{sec:pij}

The estimation of the covariance matrix $\Gamma$ is hard to perform (on the fly and simultaneously)
as the algorithm is running. We leave this problem for future research and devote here the
first $n$ rounds to this estimation.
We created from scratch the estimation of $\Gamma$ proposed below and studied in Lemma~\ref{lem:gamma},
as we could find no suitable result in the literature.  %to that end

For each pair \\[3pt]
$ \hspace*{10pt}
(i,j) \in E \defeq \big\{ (i,j) \in \{1,\dots,K\}^2: 1\leq i\leq j\leq K \big\}
$ \\[3pt]
we define the weight vector $p^{(i,j)}$ as: for $k \in \{1,\dots,K\}$, \vspace*{-3pt}
\[
	p^{(i,j)}_k = \left\{
	\begin{array}{ll}
		1 & \text{if } k = i = j, \\
		1/2 & \text{if } k \in \{i,j\} \text{ and } i\neq j, \\
		0 & \text{if } k \notin \{i,j\}.
	\end{array} \right.
\]
These correspond to all weights vectors that either assign all the mass to a single component, like the $p^{(i,i)}$,
or share the mass equally between two components, like the $p^{(i,j)}$ for $i\neq j$.
There are $K(K+1)/2$ different weight vectors considered.
We order these weight vectors, e.g., in lexicographic order, and use them one after the other, in order.
This implies that in the initial exploration phase of length $n$, each vector indexed by $E$ is selected at least \\[3pt]
	$
	\hspace*{60pt} \textstyle{n_0 \defeq \left\lfloor \frac{2n}{K(K+1)} \right\rfloor}
	$ \\[3pt]
times. At the end of the exploration period, we define  $\smash{\hat \theta_n}$ as in~\eqref{eq:hth} and the estimator \vspace*{-14pt}
\begin{equation}
	\qquad  \smash{\hat \Gamma_n \in \argmin_{\hat \Gamma \in \cM_K(\R)}\ \ \sum_{t=1}^n \big( \hat Z_t^2 - p_t^{\transp} \hat \Gamma p_t\big)^2 \,,} \mystrut(18,10)
	\label{eq:gammahat}
\end{equation}
where $\smash{\hat Z_t \defeq Y_{t,p_t} - \big[\phi(x_t,p_t)^{\transp} \hat \theta_n}\big]_C$. Note that $\smash{\hat \Gamma_n}$ can be computed efficiently by solving a linear system as soon as $K$ is small enough.

\subsection{Statement of our Main Result}

\begin{theorem}
\label{th:main}
Fix a risk level $\delta \in (0,1)$ and a time horizon $T\geq 1$. Assume that the boundedness assumptions~\eqref{eq:boundedness} hold. The optimistic algorithm~\eqref{eq:optialgo} with an initial exploration of length $n = O(T^{2/3})$ rounds
satisfies  \\[3pt]
$ \hspace*{40pt}
	\oR_T =~O\Bigl( T^{2/3} \ln^2 \big( \frac{T}{\delta}\big) \sqrt{\ln\frac{1}{\delta}} \Bigr)
$ \\[3pt]
with probability at least $1-\delta$.
\end{theorem}

When the covariance matrix $\Gamma$ is known, no initial exploration is required and
the regret bound improves to $O(\sqrt{T} \ln T)$ as far as the orders of magnitude in $T$ are concerned.
These improved rates might be achievable even if $\Gamma$ is unknown, through
a more efficient, simultaneous, estimation of $\Gamma$ and $\theta$
(an issue we leave for future research, as already mentioned at the beginning of Section~\ref{sec:pij}).

\subsection{Analysis: Structure}
\label{sec:bds}

\emph{Assumption~1: boundedness assumptions.}
They are all linked to the knowledge
that the mean consumption lies in $(0,C)$
and indicate some normalization of the modeling: \vspace*{-3pt}
\begin{equation}
\label{eq:boundedness}
\norm{\phi}_\infty \leq 1\,, \qquad \norm{\theta}_\infty \leq C\,,
\qquad \phi^{\transp} \theta \in [0,C]\,.
\end{equation}
As a consequence, $\norm{\theta} \leq \sqrt{d} \, C$
and all eigenvalues of $V_t$ lie in $[\lambda, \lambda+t]$,
thus $\ln \bigl( \det(V_t) \bigr) \in \bigl[d \ln \lambda,d\ln (\lambda+t)\bigr]$.

The deviation bound of Lemma~\ref{lm:conc-theta}
plays a key role in the algorithm. We introduce the following upper bound on it:
\begin{equation} \label{eq:Bn}
\smash{ \textstyle{B_t(\delta) \defeq
\sqrt{\lambda d} \, C + \rho \sqrt{2 \ln \frac{1}{\delta} + d \ln \big( 1 + \frac{t}{\lambda} \big)}\,.}} \mystrut(12,3)
\end{equation}
Finally, we also assume that a bound $G$ is known, such that \\[3pt]
$
	\hspace*{60pt} \smash{\forall p \in \cP, \qquad p^{\transp} \Gamma p \leq G\,.}
$ \\[3pt]
A last consequence of all these boundedness assumptions is
that $L \defeq \smash{  C^2} + G$
upper bounds the (conditionally) expected losses
$\smash{\ell_{t,p} = \bigl( \phi(x_t,p)^{\transp} \theta - c_t \bigr)^2 + p^{\transp} \Gamma p}$.

\emph{Structure of the analysis.}
The analysis exploits how well each $\hth_t$ estimates $\theta$
and how well $\hat{\Gamma}_n$ estimates $\Gamma$.
The regret bound, as is clear from Proposition~\ref{prop:1} below,
also consists of these two parts. The proof is to be found in the supplementary material (Appendix~\ref{sec:proofprop1}).

\begin{proposition}
\label{prop:1}
Fix a risk level $\delta \in (0,1)$ and an exploration budget $n \geq 2$.
Assume that the boundedness assumptions~\eqref{eq:boundedness} hold.
Consider an estimator $\hat{\Gamma}_n$ of $\Gamma$ such that
$\sup_{p \in \cP} \big|p^{\transp}(\Gamma - \hat{\Gamma}_n)p\big| \leq \gamma$ with probability at least $1 - \delta / 2$, for some $\gamma > 0$.
\smallskip \newline
Then choosing $\lambda > 0$ and \vspace*{-3pt}
\begin{align}
\nonumber
a_{t,p} & = \smash{\min \Bigl\{ L, \,\, 2C \, B_{t-1}(\delta t^{-2}) \, \bnorm{V_{t-1}^{-1/2} \phi(x_t,p)} \Bigr\}\,,} \\
\label{eq:alpha}
\alpha_{t,p} & = \gamma + a_{t,p}\,,
\end{align}
the optimistic algorithm~\eqref{eq:optialgo} ensures that w.p.\ $1-\delta$,
\[
	\smash{\textstyle{\sum_{t=n+1}^T \ell_{t,p_t}  - \sum_{t=n+1}^T \min_{p \in \cP} \ell_{t,p} \leq 2\sum_{t=n+1}^T \alpha_{t,p_t} \,.}}
\]
\end{proposition}

\begin{comm}
\citet{li2010contextual} pick $\alpha(t,p)$ proportional to $\smash{\bnorm{V_{t-1}^{-1/2} \phi(x_t,p)}}$ only,
but we need an additional term to account for the covariance matrix.
\end{comm}

We are thus left with studying how well $\hat\Gamma_n$ estimates $\Gamma$
and with controlling the sum of the $a_{t,p}$. The next two lemmas take care of these
issues. Their proofs are to be found in the supplementary material
(Appendices~\ref{app:gamma} and~\ref{sec:prooflmmain}).

\begin{lemma}
	\label{lem:gamma}
	For all $\delta \in (0,1)$,
the estimator~\eqref{eq:gammahat} satisfies: with probability at least $1-\delta$,
\vspace*{-3pt}
\begin{align*}
\smash{\sup_{p \in \cP} \left| p^{\transp} \big(\hat\Gamma_n - \Gamma\big) p \right|} & \leq (K+8) \kappa_n \sqrt{n} / n_0 = O(\kappa_n/\sqrt{n}) \\
& \quad = O \bigg( \frac{1}{\sqrt{n}} \ln^2 (n/\delta) \sqrt{\ln (1 /\delta)}\bigg)\,, \vspace{-.1cm}
\end{align*}
where we recall that $n_0 = \lfloor 2n/ (K(K+1)) \rfloor$ and \\
where $\kappa_n = \big( C + 2M_n \big) B_n(\delta/3) + M'_n$ \\
with $M_n = \rho/2+\ln(6n/\delta)$ \\
and $M'_n = M^2_n \sqrt{2\log(3K^2/\delta)} + 2\sqrt{ \exp(2\rho) \delta/6}$.
\end{lemma}

\begin{comm}
We derived the estimator of $\Gamma$ as well as Lemma~\ref{lem:gamma} from scratch:
we could find no suitable result in the literature for estimating $\Gamma$ in our context.
\end{comm}

\begin{lemma}
\label{lm:main}
No matter how the environment and provider pick the $x_t$ and $p_t$, \vspace*{-5pt}
\begin{align*}
 \textstyle{\sum_{t=n+1}^T a_{t,p_t}}
% & = \sum_{t=n+1}^T \min \Bigl\{ L, \,\, 2C \, B_{t-1}(\delta t^{-2}) \, \bnorm{V_{t-1}^{-1/2} \phi(x_t,p_t)} \Bigr\} \\
& \textstyle{\leq \sqrt{\bigl( 2C \oB \bigr)^2 + \frac{L^2}{2}} \sqrt{d T \ln \frac{\lambda+T}{\lambda}}} \\[-2pt]
&  = \textstyle{O \bigl( \sqrt{T \ln T \ln(T/\delta)} \bigr)}\,,
\end{align*}
where
$\oB \defeq \sqrt{d \lambda} \, C + \rho \sqrt{2 \ln (T^2/\delta) + d \ln (1 + T/\lambda)}$.
\end{lemma}

\begin{comm}
This lemma follows from a straightforward adaptation/generalization of Lemma~19.1 of the monograph by~\citet{lattimore2018bandit};
see also a similar result in Lemma~3 by~\citet{chu2011contextual}.
\end{comm}

We are now ready to conclude the proof of Theorem~\ref{th:main}.
Using for the first $n \geq 2$ rounds that $L = C^2 + G$ upper bounds the (conditionally) expected losses $\ell_{t,p_t}$,
Proposition~\ref{prop:1} and Lemmas~\ref{lem:gamma} and~\ref{lm:main} show that, w.p. $1-\delta$ \vspace*{-3pt}
\begin{align*}
 & \textstyle{\oR_T \leq nL + T\gamma + \sum_{t=n+1}^T a_{t,p_t}} \\
& \leq \mystrut(15,4) \smash{ \textstyle{nL + O \Big( T \ln^2 \bigl( \frac{n}{\delta} \bigr) \sqrt{\frac{\ln(1/\delta)}{n}} +  \sqrt{T\, \ln T \ln(T/\delta)} \Big) \,.}}
\end{align*}
Picking $n$ of order $T^{2/3}$ concludes the proof.

\noindent
\textbf{Case of known covariance matrix $\Gamma$:} We then have $\gamma = 0$ in Proposition~\ref{prop:1}
and we may discard Lemma~\ref{lem:gamma}. Taking $n = 2$, the obtained regret bound is
$2L + \sqrt{T\, \ln T \ln(T/\delta)}$.

\begin{comm}
The algorithm of Theorem~\ref{th:main} depends on $\delta$ via the tuning~\eqref{eq:alpha} of
$\alpha$. But we can also define
a regret with full expectations $\E\bigl[\ell_{t,p_t}\bigr]$ and $\min \E\bigl[\ell_{t,p}\bigr]$---remember
from Sections~\ref{sec:regret} and~\ref{sec:optiregret} that the losses $\ell_{t,p}$ are conditional expectations.
In that case the algorithm can be made independent of $\delta$. Only Step~3 of the proof
of Proposition~\ref{prop:1} is to be modified. The same rates in~$T$ are obtained.
\end{comm}

\section{Fast Rates, with Model~2}
\label{sec:fast}

In this section, we consider Model~2 and show that under an attainability condition stated below,
the order of magnitude of the regret bound in Theorem~\ref{th:main} can be reduced to a poly-logarithmic rate.
This kind of fast rates already exist in the literature of linear contextual bandits (see, e.g.,
\citealp{abbasi2011improved}, as well as \citealp{dani2008stochastic}) but are not so frequent.
We underline in the proof the key step where we gain orders of magnitude in the regret bound.
Before doing so, we note that similarly to~Section~\ref{sec:regret}, \vspace*{-3pt}
\begin{equation}
\label{eq:ellmod2}
\smash{\E \bigl[(Y_{t,p} - c_t)^2 \bigr] = \bigl( \phi(x_t,p)^{\transp} \theta - c_t \bigr)^2 + \sigma^2\,,} \mystrut(8,0)
\end{equation}
which leads us to introduce a regret $\oR_T$ defined by $\oR_T =$  \vspace*{-3pt}
\[
\smash{ \sum_{t=1}^T \bigl( \phi(x_t,p_t)^{\transp} \theta - c_t \bigr)^2
- \sum_{t=1}^T \min_{p \in \cP} \Bigl\{ \bigl( \phi(x_t,p)^{\transp} \theta - c_t \bigr)^2 \!\Bigr\}} \,. \mystrut(17,10)
\]
Thus, as far as the minimization of the regret is concerned, Model~2 is a special case of
Model~1, corresponding to a matrix $\Gamma$ that can be taken as the null matrix $[0]$.
Of course, as explained in Section~\ref{sec:modeling},
the covariance matrix $\Gamma$ of Model~2 is $\sigma^2 [1]$ in terms of real modeling,
but in terms of regret-minimization it can be taken as $\Gamma = [0]$.
Therefore, all results established above for Model~1 extend to Model~2, but under
an additional assumption stated below, the $T^{2/3}$ rates (up to poly-logarithmic terms)
obtained above can be reduced to poly-logarithmic rates only.

\emph{Assumption~2: attainability.}
For each time instance $t \geq 1$, the expected mean consumption is attainable, i.e., \\[-10pt]
\begin{equation} \label{eq:attainass}
	\exists p \in \cP : \quad \phi(x_t,p)^{\transp} \theta = c_t\,. \vspace{-10pt}
\end{equation}

We denote by $p_t^\star$ such an element of $\cP$.
In Model~2 and under this assumption, the expected losses $\ell_{t,p}$ defined in~\eqref{eq:ellmod2} are such that,
for all $t \geq 1$ and all $x_t \in \cX$, \vspace*{-5pt}
\begin{equation}
\label{eq:csq-att}
\smash{\min_{p \in \cP} \ell_{t,p}
= \ell_{t,p_t^\star} = \sigma^2\,.} \mystrut(0,8)
\end{equation}
\vspace{-15pt}

As in Model~2 the variance terms $\sigma^2$ cancel out when considering the regret,
the variance $\sigma^2$ does not need to be estimated. Our optimistic algorithm thus takes a simpler
form. For each $t \geq 2$ and $p \in \cP$ we consider the same estimators~\eqref{eq:hth}
of $\theta$ as before and then define \vspace*{-3pt}
\[
	\tl_{t,p} = \bigl( \phi(x_t,p)^{\transp} \hth_{t-1} - c_t \bigr)^2
\]
(no clipping needs to be considered in this case). We set \vspace*{-3pt}
\begin{equation}
\label{eq:defbeta}
\beta_{t,p} = B_{t-1}(\delta t^{-2})^2 \, \bnorm{V_{t-1}^{-1/2} \phi(x_t,p)}^2
\end{equation}
and then pick: \vspace*{-22pt}
\begin{equation}
\label{eq:optialgo2}
\hspace*{45pt} \smash{p_t \in \argmin_{p \in \cP} \Bigl\{ \tl_{t,p} - \beta_{t,p} \Bigr\}}
\end{equation}
for $t \geq 2$ and $p_1$ arbitrarily. The tuning parameter $\lambda > 0$
is hidden in $B_{t-1}(\delta t^{-2})^2$. We get the following theorem, whose proof is deferred to Appendix~\ref{sec:th2}
and re-uses many parts of the proofs of Proposition~\ref{prop:1} and Lemma~\ref{lm:main}.
Without the attainability assumption~\eqref{eq:attainass},
a regret bound of order $\sqrt{T}$ up to logarithmic terms
could still be proved.

\begin{theorem}
\label{th:2}
In Model~2, assume that the boundedness and attainability assumptions~\eqref{eq:boundedness} and~\eqref{eq:attainass} hold.
Then, the optimistic algorithm~\eqref{eq:optialgo2}, tuned with $\lambda >0$, ensures that
for all $\delta \in (0,1)$,  \\[3pt]
$\hspace*{30pt}
\smash{ \textstyle{\oR_T \leq
d \bigl( 4 \oB^2 + \frac{C^2}{2} \bigr) \ln \frac{\lambda+T}{\lambda}
= O \bigl( \ln^2(T) \bigr)\,,}} \mystrut(12,0)
$ \\[3pt]
w.p.\ at least $1-\delta$, where $\oB$ is defined as in~Lemma~\ref{lm:main}.
\end{theorem}

\section{Simulations}
\label{sec:simus}
Our simulations rely on a real data set of residential electricity consumption, in which
different tariffs were sent to the customers according to some policy.
But of course, we cannot test an alternative policy on historical data (we only
observed the outcome of the tariffs sent) and therefore need to build first a data
simulator.

\subsection{The Underlying Real Data Set / The Simulator}
\label{subsec:simulator}

We consider open data published\footnote{\textit{SmartMeter Energy Consumption Data in London
Households} -- see https://data.london.gov.uk/dataset/smartmeter-energy-use-data-in-london-households} by UK Power Networks and containing energy consumption (in kWh per half hour) at half hourly
intervals of a thousand customers subjected to dynamic energy prices. A single tariff (among High--1, Normal--2 or Low--3)
was offered to all customers for each half hour and
was announced in advance.
The report by \citet{schofield2014residential} provides a full description of this experimentation and an exhaustive analysis of results.
We only kept customers with more than $95\%$ of data available ($980$ clients) and considered their mean\footnote{Only such a level of aggregation allows
a proper estimation (individual consumptions are erratic); \citealp{Sevlian18}.} consumption.
As far as contexts are concerned, we
considered half-hourly temperatures $\tau_t$ in London, obtained from
https://www.noaa.gov/ (missing data managed by linear interpolation).
We also created calendar variables: the day of the week $w_t$
(equal to $1$ for Monday, $2$ for Tuesday, etc.),
the half-hour of the day $h_t \in \{1,\dots,48\}$,
and the position in the year: $y_t \in [0,1]$,
linear values between $y_t=0$ on January 1st at 00:00 and $y_t=1$ on December the 31st at 23:59.

\textbf{Realistic simulator.}~~It
is based on the following additive model, which breaks down time by half hours: \vspace{-.3cm}
\begin{multline}
\smash{\textstyle{\varphi(x_t,j) = \sum_{h=1}^{48} \big[f^{\tau}_h(\tau_t) + f_h^y(y_t) + \eta_{h} \big] \ind{h_t=h}} } \mystrut(10,0)\\
\textstyle{+ \sum_{w=1}^7 \zeta_{w}  \ind{w_t=w} + \xi_j\,,}
\label{eq:gam}
\end{multline}
\vspace*{-.85cm}

where the $f^{\tau}_h$ and $f^{y}_h$ are functions catching the effect of the temperature and of the yearly seasonality.
As explained in Example~\ref{ex:2}, the transfer parameter $\theta$ gathers coordinates of
the $f^{\tau}_h$ and the $f^{y}_h$ in bases of splines, as well as the coefficients $\eta_h$, $\zeta_w$ and $\xi_j$.
Here, we work under the assumption that exogenous factors do not impact customers' reaction to tariff changes
(which is admittedly a first step, and more complex models could be considered).
Our algorithms will have to sequentially estimate the parameter~$\theta$, but we also
need to set it to get our simulator in the first place. We do so by exploiting historical data
together with the allocations of prices picked, of the form $(0,1,0)$, $(1,0,0)$
and $(0,0,1)$ only on these data (all customers were getting the same tariff),
and apply the formula~\eqref{eq:gam} through the R--package \texttt{mgcv}
(which replaces the $\lambda$ identity matrix with a slightly more complex definite positive matrix $S$,
see \citealp{wood2006generalized}).
The deterministic part of the obtained model is realistic enough: its adjusted R-square on
historical observations equals $92\%$ while its mean absolute percentage of error equals $8.82\%$.
Now, as far as noise is concerned,
we take multivariate Gaussian noise vectors $\epsilon_t$,
where the covariance matrix $\Gamma$ was built again based on realistic values.
The diagonal coefficients $\Gamma_{j,j}$ are given by the empirical variance of the residuals associated with tariff~$j$,
while non-diagonal coefficients $\Gamma_{j,j'}$ are given by the empirical covariance
between residuals of tariffs $j$ and $j'$ at times $t$ and $t \pm 48$; 
this matrix $\Gamma$ has no special form, see Appendix~\ref{app:GammaEst} for more details
and its numerical expression. \vspace{-7pt}

\subsection{Experiment Design: Learning Added Effects $\xi_j$}

{\bfseries Target creation.}~~We
focus on attainable targets $c_t$, namely, $\varphi(x_t,1) \leq c_t\leq \varphi(x_t,3)$.
To smooth consumption, we pick $c_t$ near $\varphi(x_t,3)$ during the night and near $\varphi(x_t,1)$ in the evening.
These hypotheses can be seen as an ideal configuration where targets and customers portfolio are in a way compatible.

{\bfseries $\mathcal{P}$ restriction.}~~We
assume that the electricity provider cannot send Low and High tariffs at the same round and that population can be split in $N=100$ equal subsets.
Thus, $\mathcal{P}$ is restricted to the grid  \\[3pt]
$ \hspace*{3pt}
 \smash{\textstyle{\Big\{\big(\frac{i}{N},1\!-\!\frac{i}{N},0\big),\,\,\big(0,\frac{i}{N},1\!-\!\frac{i}{N}\!\big),
 \quad i \in \{0,\dots,N\} \Big\}}}
$

\textbf{Training period, testing period.}~~We
create one year of data using historical contexts and assume
that only Normal tariffs are picked at first: $p_t = (0,1,0)$; this is a training period,
which corresponds to what electricity providers are currently doing.
As they can accurately estimate the covariance matrix~$\Gamma$ by ad-hoc methods,
we assume that the algorithm knows the matrix $\Gamma$ used by the simulator.
Then the provider starts exploring the effects of tariffs for an additional month (a January month, based on the historical
contexts) and freely picks the $p_t$ according to our algorithm; this is the testing period.
The estimation of $\theta$ in this testing period is still performed via the formula~\eqref{eq:gam} and as indicated above (with
the \texttt{mgcv} package), including the year when only $p_t = (0,1,0)$ allocations were picked.
For learning to then focus on the
parameters $\xi_j$, as other parameters were decently estimated in the training period,
we modify the exploration term $\alpha_{t,p}$ of~\eqref{eq:optialgo} into \\[3pt]
$
	\hspace*{40pt} \alpha_{t,p} = \, 2 C B_{t-1}(\delta t^{-2}) \|\tilde{V}_{t-1}^{-1/2} p_t\| \,,
$ \\[3pt]
with $\tilde{V}_{t-1}= \lambda I_d + \sum_{s=1}^{t-1}p_sp_s^{\transp}$. We pick a convenient~$\lambda$.

\subsection{Results}

Algorithms were run $200$ times each.
The simplest set of results is provided in Figure~\ref{fig3}:
the regrets suffered on each run are compared to the theoretical orders of magnitude of the regret bounds.
As expected, we observe a lower regrets for Model~2.
The bottom parts of Figures~\ref{fig1}--\ref{fig2}
indicate, for a single run, which allocation vectors $p_t$ were picked over time.
During the first day of the testing period,
the algorithms explore\footnote{Note that, over the first iterations, the exploration term for Model~2 is much larger than the
exploitation term (but quickly vanishes), which leads to an initial quasi-deterministic exploration and
an erratic consumption (unlike in Model~1).} the effect of tariffs by sending the same tariff to all customers (the $p_t$
vectors are Dirac masses) while at the end of the testing period, they
cleverly exploit the possibility to split the population in two groups of tariffs.
We obtain an approximation of the expected mean consumption $\varphi(x_t,p_t)$
by averaging the $200$ observed consumptions, and this is the main (black, solid)
line to look at in the top parts of Figures~\ref{fig1}--\ref{fig2}.
Four plots are depicted depending on the day of the testing period (first, last) and of the model
considered. These (approximated) expected mean consumptions may be compared to
the targets set (dashed red line). The algorithms seem to perform better on
the last day of the testing period for Model~2 than for Model~1 as the expected mean consumption
seems closer to the target.
However, in Model~1, the algorithm has to pick tariffs leading to the best bias-variance trade-off
(the expected loss features a variance term). This is why the average consumption does not overlap the target as in Model~2.
This results in a slightly biased estimator of the mean consumption in Model~1.

\begin{figure*}[p]
\centering
\includegraphics[width=.42\textwidth]{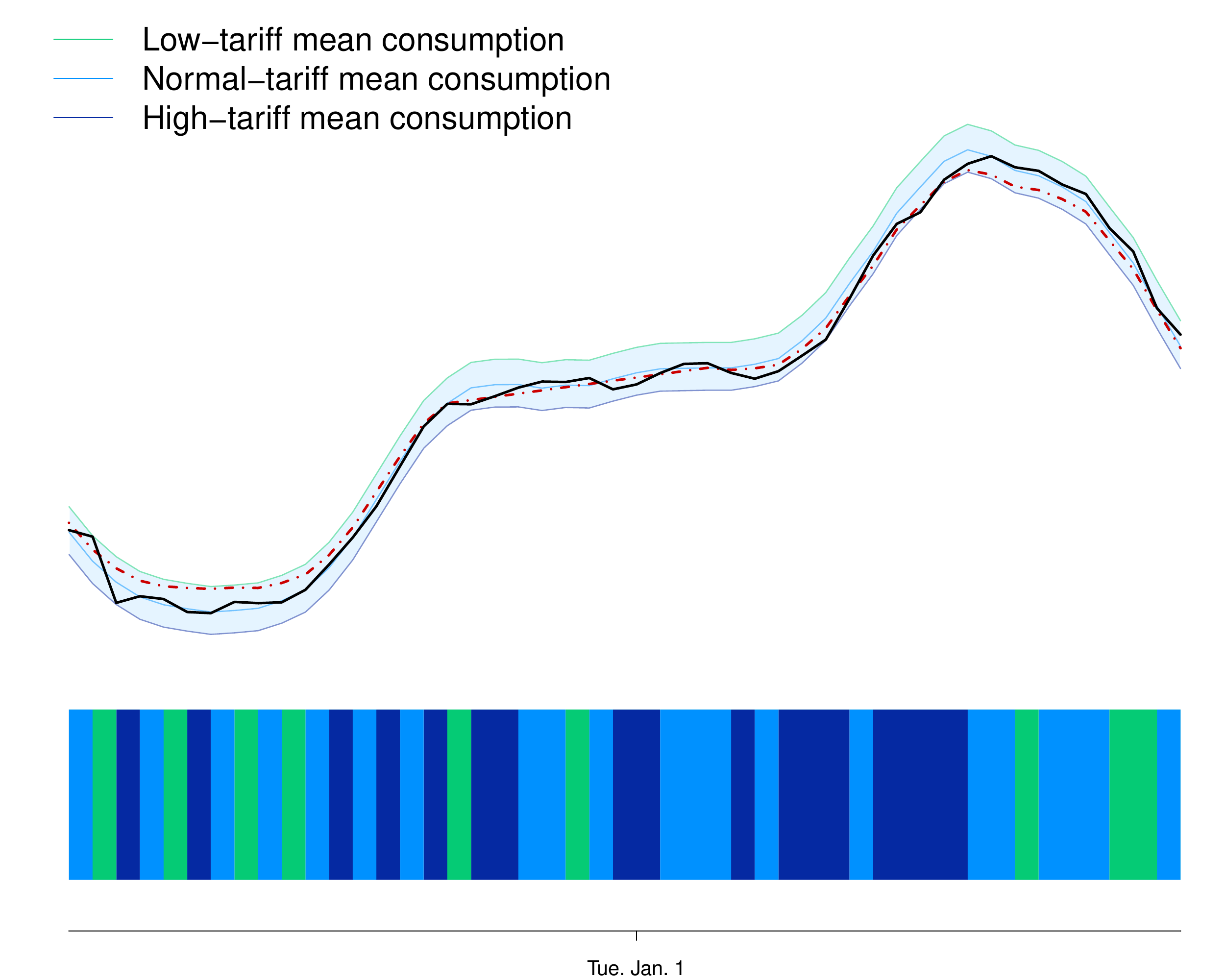}
\includegraphics[width=.42\textwidth]{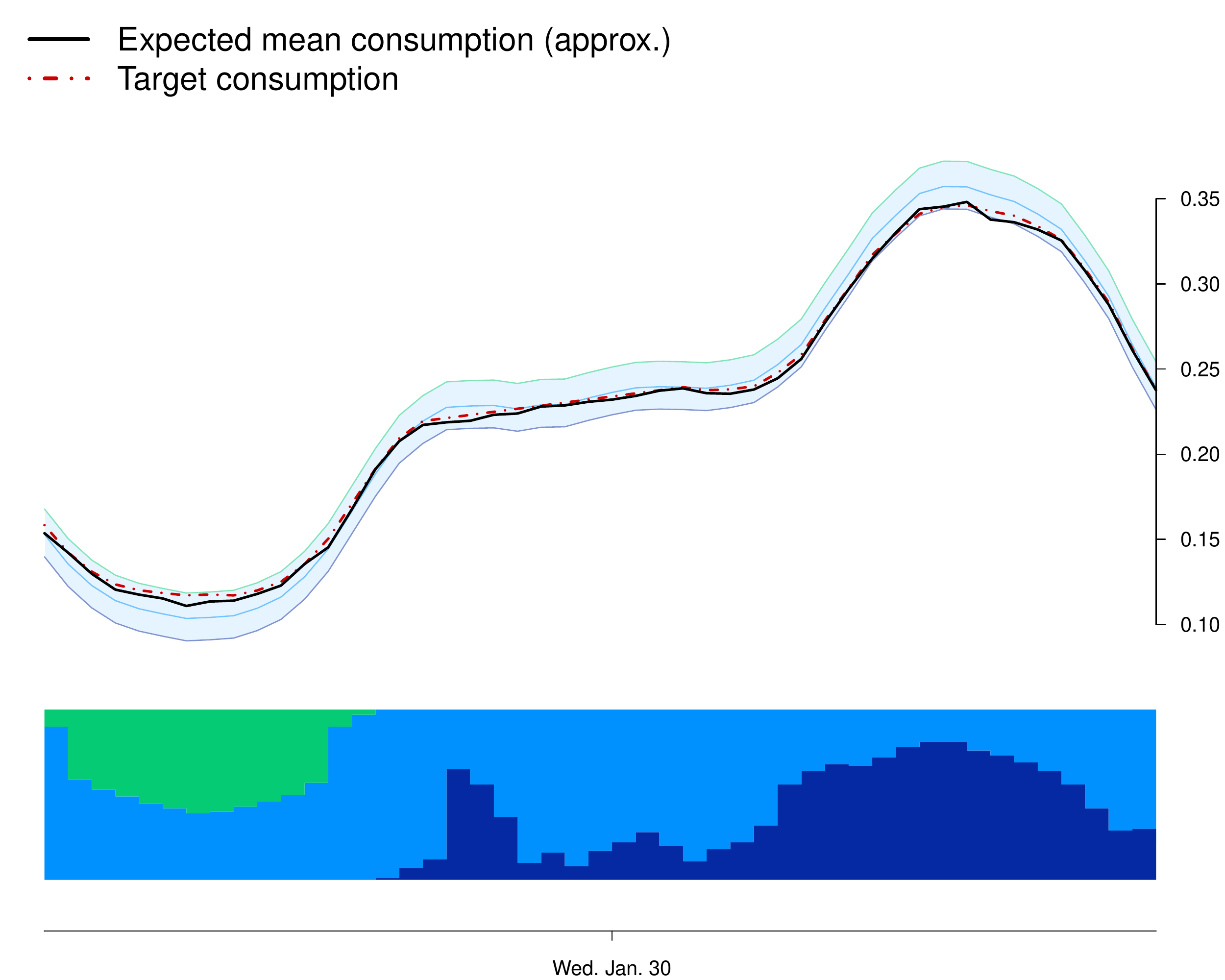}
\caption{\qquad \emph{Left}: January 1st (first day of the testing set). \qquad \qquad \emph{Right}: January 30th (last day of the testing set). \smallskip \\
\emph{Top}: $200$ runs are considered. Plot: average of mean consumptions over $200$ runs for the algorithm associated with Model~1 (full black line);
target consumption (dashed red line); mean consumption associated with each tariff (Low--1 in green, Normal--2 in blue and High--3 in navy).
The envelope of attainable targets is in pastel blue. \smallskip \\
\emph{Bottom}: A single run is considered. Plot: proportions $p_t$ used over time.}
\label{fig1}
\end{figure*}

\begin{figure*}[p]
\centering
\includegraphics[width=.42\textwidth]{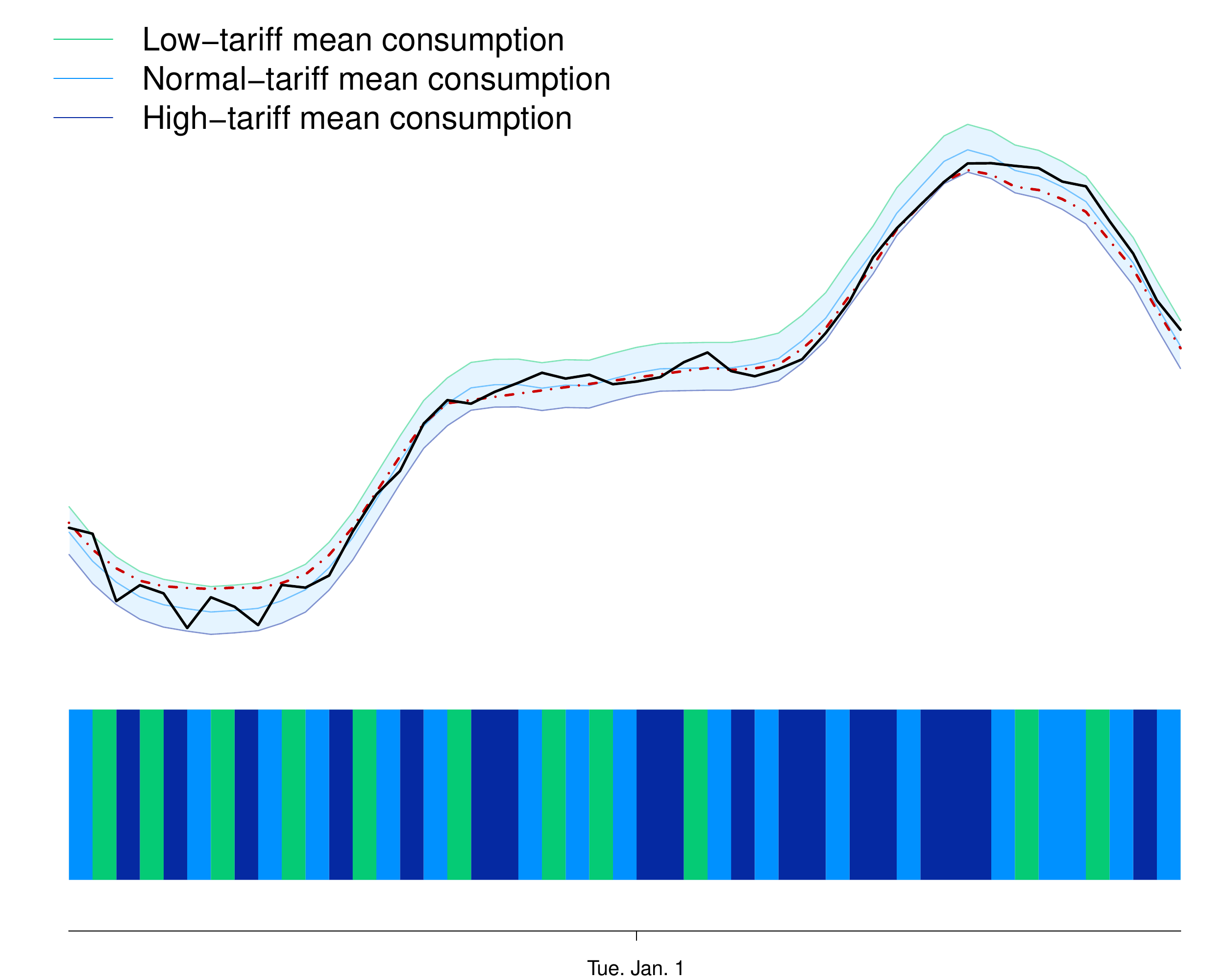}
\includegraphics[width=.42\textwidth]{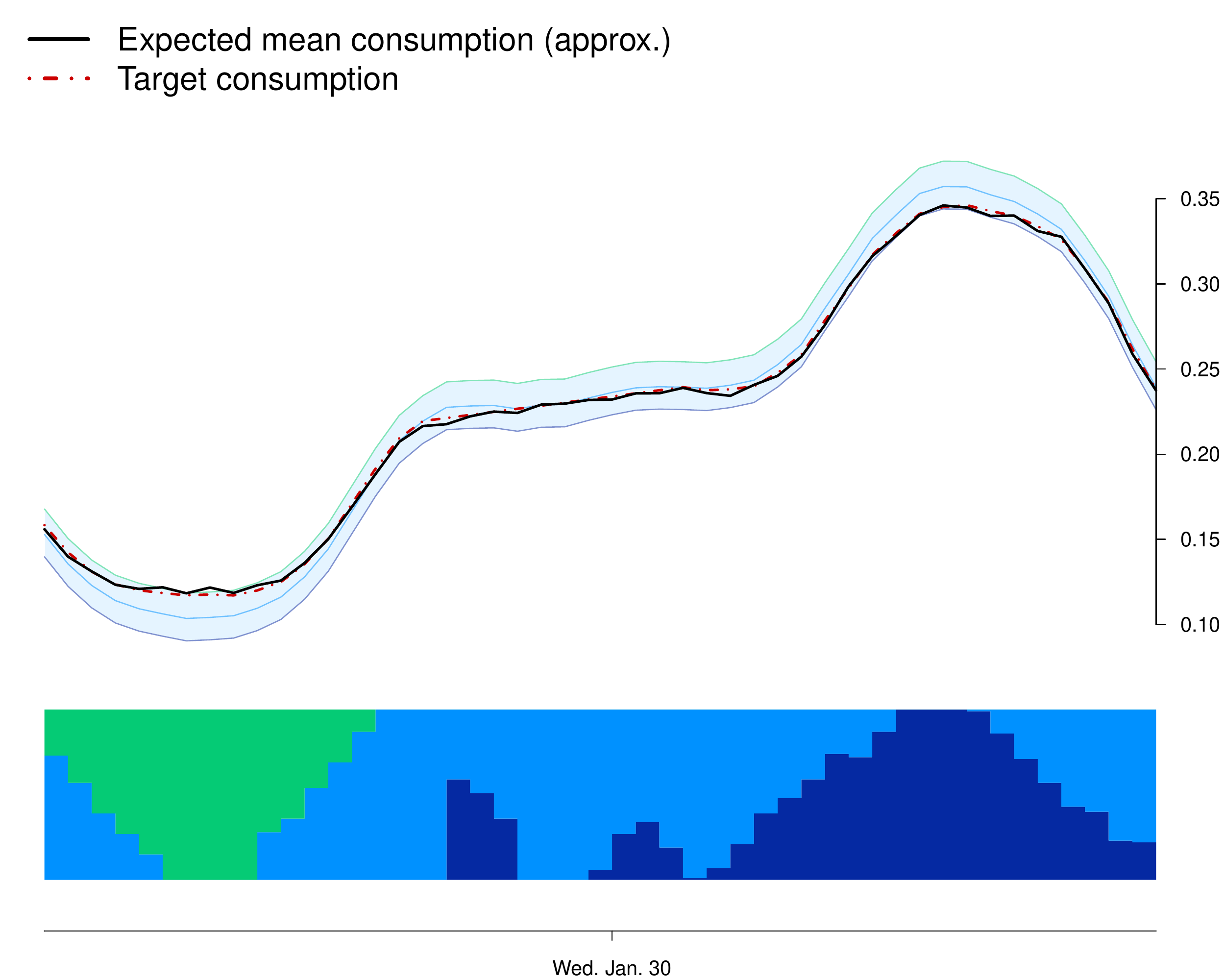}
\caption{Same legend, but with Model~2 (full black line). \hfill \ \\ \ \\ }
\label{fig2}
\end{figure*}

\begin{figure*}[p]
\centering
\includegraphics[width=.42\textwidth]{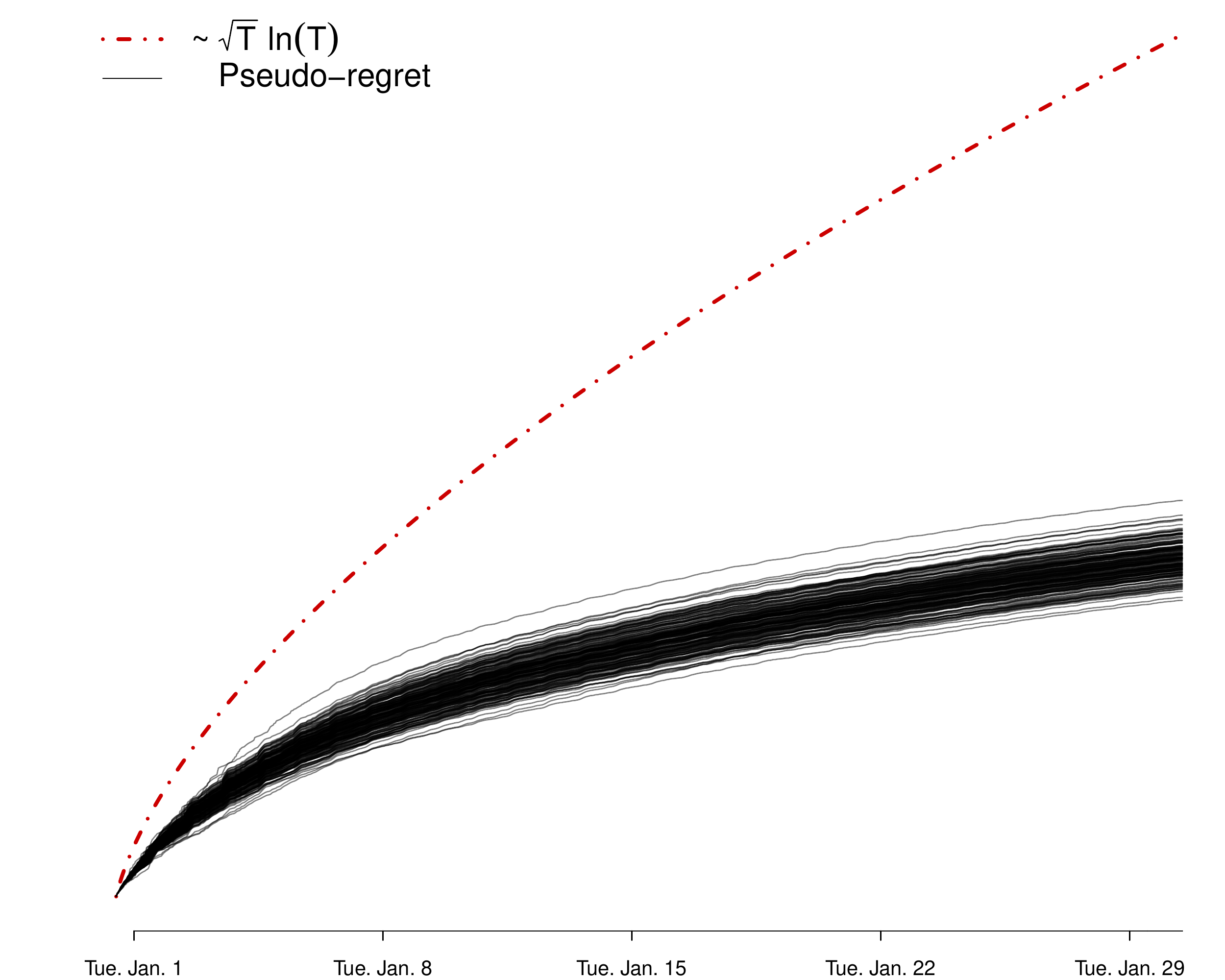}
\includegraphics[width=.42\textwidth]{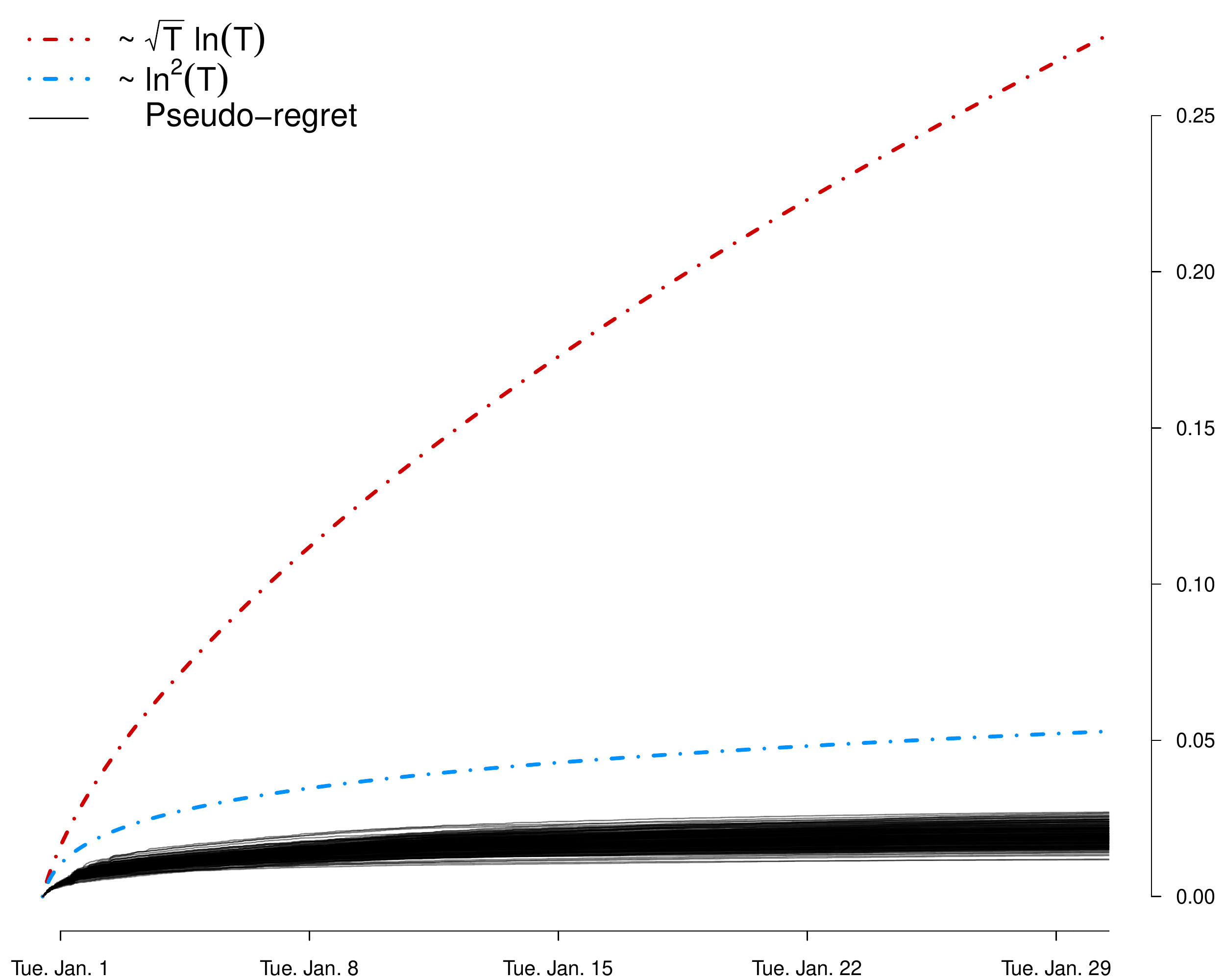}
\caption{Regret curves for each of the $200$ runs for Model~1 (\emph{left}) and Model~2 (\emph{right}).
We also provide plots of $c \sqrt{T} \ln T$ and $c' \ln^2(T)$ for some well-chosen constants $c,c' > 0$;
these are the rates to be considered as the covariance matrix $\Gamma$ is assumed to be known.}
\label{fig3}
\end{figure*}

\newpage
\bibliography{BGGS19-PowerConsumptionControl-Biblio}
\bibliographystyle{icml2019}

\newpage
\twocolumn[
\icmltitle{Target Tracking for Contextual Bandits: \\
           Application to Demand Side Management\\ \ \\ Supplementary material}

\begin{center}
\textbf{Margaux Br{\'e}g{\`e}re \qquad Pierre Gaillard \qquad Yannig Goude \qquad Gilles Stoltz}
\end{center}

\vskip 0.3in
]

\appendix

We provide the proofs in
order of appearance of the corresponding result: \smallskip \newline
\phantom{esp}-- The proof of Lemma~\ref{lm:conc-theta} in Appendix~\ref{app:conc-theta} \newline
\phantom{esp}-- The proof of Proposition~\ref{prop:1} in Appendix~\ref{sec:proofprop1} \newline
\phantom{esp}-- The proof of Lemma~\ref{lem:gamma} in Appendix~\ref{app:gamma} \newline
\phantom{esp}-- The proof of Lemma~\ref{lm:main} in Appendix~\ref{sec:prooflmmain} \newline
\phantom{esp}-- The proof of Theorem~\ref{th:2} in Appendix~\ref{sec:th2} \smallskip \newline
We also give more details on the numerical expression of the covariance matrix $\Gamma$ built in the experiments (see Section~\ref{subsec:simulator})
based on real data: \smallskip \newline
\phantom{esp}-- Details on the covariance matrix $\Gamma$ in Appendix~\ref{app:GammaEst}. \newline

\newpage
\section{Proof of Lemma~\ref{lm:conc-theta}}
\label{app:conc-theta}

The proof below relies on Laplace's method on super-martingales, which is a standard argument
to provide confidence bounds on a self-normalized sum of conditionally centered random vectors.
See Theorem~2 of \citet{abbasi2011improved} or Theorem~20.2 in the monograph by \citet{lattimore2018bandit}.
Under Model~1 and given the definition of $V_t$, we have the rewriting 
\begin{align*}
\hat{\theta}_t &= V_t^{-1}\sum_{s=1}^t \phi(x_s,p_s) Y_{s,p_s} \\
& = V_t^{-1} \sum_{s=1}^t \phi(x_s,p_s)\big(\phi(x_s,p_s)^{\transp} \theta + p_s^{\transp} \epsilon_s \big) \\
&  = V_t^{-1} \big( (V_t - \lambda \Id) \theta + M_t \big)
= \theta - \lambda V_t^{-1} \theta + V_t^{-1} M_t\,,
\end{align*}
where we introduced
\[
M_t = \sum_{s=1}^t \phi(x_s,p_s) p_s^{\transp} \epsilon_s\,,
\]
which is a martingale with respect to $\cF_t = \sigma(\epsilon_1,\ldots,\epsilon_t)$.
Therefore, by a triangle inequality,
\begin{align*}
\bigl\Arrowvert V_t^{1/2} \bigl( \hth_t - \theta \bigr) \bigr\Arrowvert
& = \bigl\Arrowvert - \lambda V_t^{-1/2} \theta + V_t^{-1/2} M_t \Arrowvert \\
& \leq  \lambda \bigl\Arrowvert V_t^{-1/2} \theta \bigr\Arrowvert
+ \bigl\Arrowvert V_t^{-1/2} M_t \bigr\Arrowvert \,.
\end{align*}
On the one hand, given that all eigenvalues of the symmetric
matrix $V_t$ are larger than $\lambda$ (given the $\lambda \Id$
term in its definition), all eigenvalues of $V_t^{-1/2}$ are
smaller than $1/\sqrt{\lambda}$ and thus,
\[
\lambda \bigl\Arrowvert V_t^{-1/2} \theta \bigr\Arrowvert
\leq \lambda \, \frac{1}{\sqrt{\lambda}} \Arrowvert\theta\Arrowvert
= \sqrt{\lambda} \Arrowvert\theta\Arrowvert\,.
\]
We now prove, on the other hand, that with probability at least $1-\delta$,
\[
\bigl\Arrowvert V_t^{-1/2} M_t \bigr\Arrowvert \leq
\rho \sqrt{2 \ln \frac{1}{\delta} + d \log \frac{1}{\lambda} + \log \det(V_t)}\,,
\]
which will conclude the proof of the lemma.

\emph{Step 1: Introducing super-martingales.} For all
$\nu \in \R^d$, we consider
\[
S_{t,\nu} = \exp\biggl( \nu^{\transp} M_t - \frac{\rho^2}{2} \nu^{\transp} V_t \nu \biggr)
\]
and now show that it is an $\cF_t$--super-martingale.
First,
note that since the common distribution of the $\epsilon_1,\epsilon_2,\ldots$ is
$\rho$--sub-Gaussian, then for all $\cF_{t-1}$--measurable random vectors~$\nu_{t-1}$,
\begin{equation}
\label{eq:subG-csq}
\E \Bigl[ \e^{\nu_{t-1}^{\transp} \epsilon_t} \, \Big| \, \cF_{t-1} \Bigr] \leq \e^{\rho^2 \norm{\nu_{t-1}}^2 /2}\,.
\end{equation}
Now,
\begin{multline*}
S_{t,\nu} = S_{t-1,\nu} \,
\exp\biggl( \nu^{\transp} \phi(x_t,p_t) p_t^{\transp} \epsilon_t \\ - \frac{\rho^2}{2} \nu^{\transp} \phi(x_t,p_t) \phi(x_t,p_t)^{\transp} \nu \biggr)
\end{multline*}
where, by using the sub-Gaussian assumption~\eqref{eq:subG-csq} and the fact that
$\sum_j p_{j,t}^2 \leq 1$ for all convex weight vectors $p_t$,
\vspace*{-10pt}
\begin{multline*}
 \E \Bigl[ \exp\bigl( \nu^{\transp} \phi(x_t,p_t) p_t^{\transp} \epsilon_t \, \Big| \, \cF_{t-1} \Bigr] \\
\leq \  \exp\biggl( \frac{\rho^2}{2} \nu^{\transp} \phi(x_t,p_t) \smash{\underbrace{p_t^{\transp} p_t}_{\leq 1}} \phi(x_t,p_t)^{\transp} \nu \biggr) \,.
\end{multline*}
This implies $\E\big[S_{t,\nu}\big|\cF_{t-1}\big] \leq S_{t-1,\nu}$.

Note that the rewriting of $S_{t,\nu}$ in its vertex form is,
with $m = V_t^{-1} M_t / \rho^2$:
\begin{align*}
S_{t,\nu} & = \exp\biggl( \frac{1}{2} (\nu - m)^{\transp} \, \rho^2 V_t \, (\nu - m) + \frac{1}{2} m^{\transp} \rho^2 V_t \, m \biggr) \\
& = \exp\biggl( \frac{1}{2} (\nu - m)^{\transp} \, \rho^2 V_t \, (\nu - m) \biggr) \\
& \qquad \qquad \qquad \times \exp \biggr( \frac{1}{2\rho^2} \bigl\Arrowvert V_t^{-1/2} M_t \bigr\Arrowvert^2 \biggr).
\end{align*}

\emph{Step 2: Laplace's method---integrating $S_{t,\nu}$ over $\nu \in \R^d$.}
The basic observation behind this method is that (given the vertex form)
$S_{t,\nu}$ is maximal at $\nu = m = V_t^{-1} M_t / \rho^2$ and then equals
$\exp\bigl( \bigl\Arrowvert V_t^{-1/2} M_t \bigr\Arrowvert^2/(2\rho^2) \bigr)$,
which is (a transformation of) the quantity to control. Now, because the $\exp$ function
quickly vanishes, the integral over $\nu \in \R^d$ is close to this maximum.
We therefore consider
\[
\oS_t = \int_{\R^d} S_{t,\nu} \d\nu\,.
\]
We will make repeated uses of the fact that the Gaussian density functions,
\[
\nu \longmapsto \frac{1}{\sqrt{\det(2\pi C)}} \, \exp \biggl( (\nu - m)^{\transp} C^{-1} (\nu - m) \biggr),
\]
where $m \in \R^d$ and $C$ is a (symmetric) positive-definite matrix, integrate to~$1$ over $\R^d$.
This gives us first the rewriting
\[
\oS_t = \sqrt{\det\bigl( 2 \pi \rho^{-2} V_t^{-1} \bigr)} \, \exp \biggr( \frac{1}{2\rho^2} \bigl\Arrowvert V_t^{-1/2} M_t \bigr\Arrowvert^2 \biggr).
\]
Second, by the Fubini-Tonelli theorem and the super-martingale property
\[
\E\bigl[S_{t,\nu}\bigr] \leq \E\bigl[S_{0,\nu}\bigr] = \exp \bigl( - \lambda \rho^2 \norm{\nu}^2 / 2 \bigr)\,,
\]
we also have
\begin{multline*}
\E\bigl[\oS_t\bigr] \leq \int_{\R^d} \exp \bigl( - \lambda \rho^2 \norm{\nu}^2 / 2 \bigr) \d\nu \\
= \sqrt{\det\bigl( 2\pi \rho^{-2} \lambda^{-1} \Id \bigr)}\,.
\end{multline*}
Combining the two statements, we proved
\[
\E \Biggl[ \exp \biggr( \frac{1}{2\rho^2} \bigl\Arrowvert V_t^{-1/2} M_t \bigr\Arrowvert^2 \biggr) \Biggr]
\leq \sqrt{\frac{\det\bigl( V_t \bigr)}{\lambda^d}}\,.
\]

\emph{Step 3: Markov-Chernov bound.} For $u > 0$,
\begin{align*}
\lefteqn{\P \Bigl[ \bigl\Arrowvert V_t^{-1/2} M_t \bigr\Arrowvert > u \Bigr]} \\
& = \P \biggl[ \frac{1}{2\rho^2} \bigl\Arrowvert V_t^{-1/2} M_t \bigr\Arrowvert^2 > \frac{u^2}{2\rho^2} \biggr] \\
& \leq \exp \biggl( - \frac{u^2}{2\rho^2} \biggr) \, \E \Biggl[ \exp \biggr( \frac{1}{2\rho^2} \bigl\Arrowvert V_t^{-1/2} M_t \bigr\Arrowvert^2 \biggr) \Biggr] \\
& \leq \exp \biggl( - \frac{u^2}{2\rho^2} + \frac{1}{2} \ln \frac{\det\bigl( V_t \bigr)}{\lambda^d} \biggr) = \delta
\end{align*}
for the claimed choice
\[
u = \rho \sqrt{2 \ln \frac{1}{\delta} + d \ln \frac{1}{\lambda} + \ln \det(V_t)}\,.
\]

\newpage
\section{Proof of Proposition~\ref{prop:1}}
\label{sec:proofprop1}

\begin{comm}
The main difference with the regret analysis of LinUCB provided by~\citet{chu2011contextual} or ~\citet{lattimore2018bandit} is in the first
part of \emph{Step~1}, as we need to deal with slightly more complicated quantities: not just with linear quantities of
the form $\phi(x_t,p)^{\transp} \theta$. Steps~2 and~3 are easy consequences of Step~1.
\end{comm}

We show below (\emph{Step~1}) that for all $t \geq 2$, if
\begin{multline}
\label{eq:bothevents}
\bigl\Arrowvert V_{t-1}^{1/2} \bigl( \hth_{t-1} - \theta \bigr) \bigr\Arrowvert \leq B_{t-1}(\delta t^{-2}) \\
\mbox{and} \qquad
\bnorm{\Gamma - \hat{\Gamma}_t}_\infty \leq \gamma\,,
\end{multline}
then
\begin{equation}
\label{eq:estim-ell}
\forall p \in \cP, \qquad \bigl| \ell_{t,p} - \hl_{t,p} \bigl| \leq \alpha_{t,p}\,.
\end{equation}
Property~\eqref{eq:estim-ell}, for those $t$ for which it is satisfied, entails (\emph{Step~2}) that
the corresponding instantaneous regrets are bounded by
\[
r_t \defeq \ell_{t,p_t} - \min_{p \in \cP} \ell_{t,p} \leq 2 \alpha_{t,p_t}\,.
\]
It only remains to deal (\emph{Step~3}) with the rounds $t$ when~\eqref{eq:estim-ell}
does not hold; they account for the $1-\delta$ confidence level.

\emph{Step~1: Good estimation of the losses.} When the two
events~\eqref{eq:bothevents} hold, we have
\begin{align*}
\lefteqn{\bigl| \ell_{t,p} - \hl_{t,p} \bigr|} \\
& = \biggl| \bigl( \phi(x_t,p)^{\transp} \theta - c_t \bigr)^2 + p^{\transp} \Gamma p \\
& \qquad - \Bigl( \bigl[ \phi(x_t,p)^{\transp} \hth_{t-1} \bigr]_C - c_t \Bigr)^2 + p^{\transp} \hat{\Gamma}_t p \biggr| \\
& \leq \bigl| p^{\transp} \Gamma p - p^{\transp} \hat{\Gamma}_t p \bigr| \\
& \qquad + \biggl| \bigl( \phi(x_t,p)^{\transp} \theta - c_t \bigr)^2
- \Bigl( \bigl[ \phi(x_t,p)^{\transp} \hth_{t-1} \bigr]_C - c_t \Bigr)^2 \biggr|.
\end{align*}
On the one hand, $\bigl| p^{\transp} \Gamma p - p^{\transp} \hat{\Gamma}_t p \bigr| \leq \gamma$
while on the other hand,
\begin{align*}
& \biggl| \bigl( \phi(x_t,p)^{\transp} \theta - c_t \bigr)^2
- \Bigl( \bigl[ \phi(x_t,p)^{\transp} \hth_{t-1} \bigr]_C - c_t \Bigr)^2 \biggr| \\
& = \Bigl| \phi(x_t,p)^{\transp} \theta - \bigl[ \phi(x_t,p)^{\transp} \hth_{t-1} \bigr]_C \Bigr| \\
& \qquad \times \Bigl| \phi(x_t,p)^{\transp} \theta + \bigl[ \phi(x_t,p)^{\transp} \hth_{t-1} \bigr]_C - 2 c_t\Bigr|\,,
\end{align*}
where by the boundedness assumptions~\eqref{eq:boundedness},
all quantities in the final inequality lie in $[0,C]$, thus
\[
\Bigl| \phi(x_t,p)^{\transp} \theta + \bigl[ \phi(x_t,p)^{\transp} \hth_{t-1} \bigr]_C - 2 c_t\Bigr| \leq 2C\,.
\]
Finally,
\begin{align}
& \Bigl| \phi(x_t,p)^{\transp} \theta - \bigl[ \phi(x_t,p)^{\transp} \hth_{t-1} \bigr]_C \Bigr| \nonumber \\
& \leq \bigl| \phi(x_t,p)^{\transp} \theta - \phi(x_t,p)^{\transp} \hth_{t-1} \bigr| \nonumber \\
\label{eq:CS-phitheta}
& \leq \Bnorm{V_{t-1}^{1/2} \bigl( \theta - \hth_{t-1} \bigr)} \, \bnorm{V_{t-1}^{-1/2} \phi(x_t,p)}\,,
\end{align}
where we used the Cauchy-Schwarz inequality for the second inequality, and the fact
that $\bigl| y - [x]_C \bigr| \leq |y - x|$ when $y \in [0,C]$ and $x \in \R$ for the first inequality.
Collecting all bounds together, we proved
\begin{multline*}
\biggl| \bigl( \phi(x_t,p)^{\transp} \theta - c_t \bigr)^2
- \Bigl( \bigl[ \phi(x_t,p)^{\transp} \hth_{t-1} \bigr]_C - c_t \Bigr)^2 \biggr| \\
\leq 2C \, \underbrace{\Bnorm{V_{t-1}^{1/2} \bigl( \theta - \hth_{t-1} \bigr)}}_{\leq B_{t-1}(\delta t^{-2})} \, \bnorm{V_{t-1}^{-1/2} \phi(x_t,p)}\,,
\end{multline*}
but of course, this term is also bounded by the quantity $L$ introduced in Section~\ref{sec:bds}.
This concludes the proof of the claimed inequality~\eqref{eq:estim-ell}.

\emph{Step~2: Resulting bound on the instantaneous regrets.}
We denote by
\begin{equation}
\label{eq:pstar}
p_t^\star \in \argmin_{p \in \cP} \bigl\{ \ell_{t,p} + p^{\transp} \Gamma p \bigr\}
\end{equation}
an optimal convex vector to be used at round~$t$. By definition~\eqref{eq:optialgo}
of the optimistic algorithm, we have that the played $p_t$ satisfies
\begin{align*}
\hl_{t,p_t} - \alpha_{t,p_t} & \leq \hl_{t,p_t^\star} - \alpha_{t,p_t^\star}\,, \\
\qquad \mbox{that is}, \qquad
\hl_{t,p_t} - \hl_{t,p_t^\star} & \leq \alpha_{t,p_t} - \alpha_{t,p_t^\star}\,.
\end{align*}
Now, for those $t$ for which both events~\eqref{eq:bothevents} hold,
the property~\eqref{eq:estim-ell} also holds and yields,
respectively for $p = p_t$ and $p = p_t^\star$:
\[
\ell_{t,p_t} - \hl_{t,p_t} \leq \alpha_{t,p_t}
\qquad \mbox{and} \qquad
\hl_{t,p^\star_t} - \ell_{t,p^\star_t} \leq \alpha_{t,p^\star_t}\,.
\]
Combining all these three inequalities together, we proved
\begin{align*}
r_t & = \ell_{t,p_t} - \ell_{t,p^\star_t} \\
& = \bigl( \ell_{t,p_t} - \hl_{t,p_t} \bigr)
+ \bigl( \hl_{t,p_t} - \hl_{t,p_t^\star} \bigr)
+ \bigl( \hl_{t,p_t^\star} - \ell_{t,p^\star_t} \bigr) \\
& \leq \alpha_{t,p_t} + (\alpha_{t,p_t} - \alpha_{t,p_t^\star}) + \alpha_{t,p^\star_t} = 2\alpha_{t,p_t}\,,
\end{align*}
as claimed. This yields the $2 \sum \alpha_{t,p_t}$ in the regret bound, where the sum is for $t \geq n+1$.

\emph{Step~3: Special cases.}
We conclude the proof by dealing with the time steps $t \geq n+1$ when at least one of the events~\eqref{eq:bothevents}
does not hold.
By a union bound, this happens for some $t \geq n+1$ with probability at most
\[
\frac{\delta}{2} + \delta \sum_{t \geq n+1} t^{-2} \leq \frac{\delta}{2} + \delta \int_{2}^\infty \frac{1}{t^2} \d t = \delta\,,
\]
where we used $n \geq 2$.
These special cases thus account for the claimed $1-\delta$ confidence level.

\newpage
\section{Proof of Lemma~\ref{lem:gamma}}
\label{app:gamma}

We derived the proof scheme below from scratch as we could find no
suitable result in the literature for estimating $\Gamma$ in our context.

We first consider the following auxiliary result.

\begin{lemma}
\label{lem:gammahat1}
Let $n\geq 1$. Assume that the common distribution of the $\epsilon_1,\epsilon_2,\ldots$ is
$\rho$--sub-Gaussian. Then, no matter how the provider picks the $p_t$, we have, for all $\delta \in (0,1)$, with probability at least~$1-\delta$,
	\[
		\left\|  \sum_{t=1}^n p_tp_t^{\transp} \big(\hat \Gamma_n -\Gamma\big) p_tp_t^{\transp} \right\|_{\infty} \leq \kappa_n \sqrt{n} \,,
	\]
	where the quantities $\kappa_n$, $M_n$ and $M'_n$ are defined as in Lemma~\ref{lem:gamma}:
\begin{align*}
M_n & \defeq \rho/2+\ln(6n/\delta) \\
M'_n & \defeq M^2_n \sqrt{2\log(3K^2/\delta)} + 2\sqrt{ \exp(2\rho) \delta/6} \\
\kappa_n & \defeq \big( C + 2M_n \big) B_n(\delta/3) + M'_n
\end{align*}
\end{lemma}

\begin{proof}[Proof of Lemma~\ref{lem:gammahat1}]
We can show that $\hat \Gamma_n$ defined in~\eqref{eq:gammahat} satisfies
 \begin{equation}
 	\label{eq:linearsystem}
\sum_{t=1}^n p_tp_t^{\transp} \hat \Gamma_n p_tp_t^{\transp} =  \sum_{t=1}^n  \hat Z_t^2 p_tp_t^{\transp} \,,
 \end{equation}
where we recall that $\hat Z_t \defeq Y_{t,p_t} - \big[\phi(x_t,p_t)^{\transp} \hat \theta_n\big]_C $.
Indeed, with,
\[
\Phi\bigl(\hat\Gamma \bigr) \defeq \sum_{t=1}^n \Big( \hat Z_t^2 - p_t^{\transp} \hat\Gamma p_t \Big)^2  =\sum_{t=1}^n \Big( \hat Z_t^2 - \Tr \big(\hat\Gamma p_t p_t^{\transp} )\Big)^2,
\]
using $\nabla_{\! A} \Tr (AB) = B$, we get
\[
\nabla_{\hat\Gamma} \Phi\bigl(\hat\Gamma\bigr) = \sum_{t=1}^n 2p_tp_t^{\transp} \Big( \hat Z_t^2 - p_t^{\transp} \hat\Gamma p_t \Big),
\]
which leads to~\eqref{eq:linearsystem} by canceling the gradient and keeping in mind
that $p_t^{\transp} \hat\Gamma p_t$ is a scalar value.

Let us denote
\[
Z_t \defeq Y_{t,p_t} - \phi(x_t,p_t)^{\transp} \theta = p_t^{\transp} \varepsilon_t
\]
for all $t\geq 1$. To prove the lemma, we replace $\hat \Gamma_n$ by using~\eqref{eq:linearsystem} and apply a triangular inequality:
\begin{align}
\label{eq:twoterms}
& \bigg\|   \sum_{t=1}^n p_tp_t^{\transp}   \big(\hat \Gamma_n -\Gamma\big) p_tp_t^{\transp}  \bigg\|_{\infty} \\
\nonumber
& \leq \bigg\|  \sum_{t=1}^n  (\hat Z_t^2 - Z_t^2) p_tp_t^{\transp}   \bigg\|_\infty
 + \left\|\sum_{t=1}^n  Z_t^2 p_tp_t^{\transp}  - p_tp_t^{\transp}  \Gamma p_tp_t^{\transp}  \right\|_{\infty}
\end{align}
We will consecutively provide bounds for each of the two terms in the right-hand side of the above inequality,
each holding with probability at least $1-\delta/3$.
To do so, we focus on the event defined below where all $Z_t$ are bounded:
\begin{equation}
\mathcal{E}_n (\delta) \defeq \bigl\{ \forall t=1,\dots n, \quad |Z_t| \leq M_n \bigr\},
\label{eq:event}
\end{equation}
with $M_n$ defined in the statement of the lemma.
We will show below that $\mathcal{E}_n (\delta)$ takes place with probability at least $1 - \delta/3$.
All in all, our obtained global bound will hold with probability at least $1-\delta$,
as stated in the lemma.

\medskip
\emph{Bounding the probability of the event $\mathcal{E}_n (\delta)$.}
Recall that $p_t$ is $\cF_{t-1} = \sigma(\epsilon_1,\ldots,\epsilon_{t-1})$ measurable. For $t\in \{1,\dots,n\}$, as $\varepsilon_t$ is a $\rho$--sub-Gaussian variable independent of $\cF_{t-1}$,
\begin{align*}
 \mathbb{E}\Bigl[ \exp(p_t^{\transp} \varepsilon_t) \,\Big|\, \mathcal{F}_{t-1}\Bigr] \leq \exp \biggl( \frac{\rho\|p_t\|^2}{2}\biggr)
\leq \exp \! \Big(\frac{\rho}{2}\Big)\,;
\end{align*}
see Footnote~\ref{fn:sG} for a reminder of the definition of a $\rho$--sub-Gaussian variable.
Using the Markov-Chernov inequality, we obtain
\begin{align}
\nonumber
\mathbb{P} \big( Z_t \geq M_n \,\big|\, \mathcal{F}_{t-1} \big) & \leq \mathbb{E} \Bigl[ \exp(Z_t) \,\Big|\, \mathcal{F}_{t-1} \Bigr] \, \exp (-M_n )  \\
 \label{eq:6n}
 &\leq \exp\Big(\frac{\rho}{2} -M_n \Big)  = \frac{\delta}{6n}.
\end{align}
Symmetrically, we get that $\mathbb{P} ( Z_t \leq -M_n  ) \leq \delta/6n.$
Combining all these bounds for $t=1,\dots,n$, the event $\mathcal{E}_n (\delta)$ happens with probability at least $1-\delta/3$.

\medskip
\emph{Upper bound on the first term in~\eqref{eq:twoterms}.}
By Assumption~\eqref{eq:boundedness}, we have
$\phi(x_t,p_t)^{\transp}\theta \in [0,C]$, thus
\[
|\hat Z_t - Z_t| = \Big| \phi(x_t,p_t)^{\transp}\theta - \big[  \phi(x_t,p_t)^{\transp} \hat\theta_n \big]_C  \Big| \leq C\,,
\]
and therefore, on $\mathcal{E}_n (\delta)$,
\[
\big|\hat Z_t +Z_t\big|\leq \big|\hat Z_t - Z_t\big|+ \big|2 Z_t\big| \leq C+2M_n\defeq M''_n\,.
\]
Noting that all components of $p_tp_t^{\transp} $ are upper bounded by $1$,
	\begin{align*}
		 \bigg\| \sum_{t=1}^n (\hat Z_t^2  & -   Z_t^2) p_tp_t^{\transp}  \bigg\|_\infty
			 \leq   \sum_{t=1}^n \big|\hat Z_t^2 - Z_t^2\big| \nonumber \\
			& = \sum_{t=1}^n  \big|(\hat Z_t - Z_t)(\hat Z_t + Z_t) \big| \nonumber \\
		 & \leq M''_n \sqrt{n \sum_{t=1}^n (\hat Z_t - Z_t)^2} \,, \nonumber
	\end{align*}
where the last inequality was obtained by $|\hat Z_t + Z_t|\leq M''_n$ together with the Cauchy-Schwarz inequality.
Using that $\bigl| y - [x]_C \bigr| \leq |y - x|$ when $y \in [0,C]$ and $x \in \R$,
we note that
\[
	\bigl| \hat Z_t - Z_t \bigr| \leq \Bigl| \phi(x_t,p_t)^{\transp}  (\hat \theta_n - \theta) \Bigr| \,.
\]
All in all, we proved so far
\begin{align*}
	 \bigg\| & \sum_{t=1}^n (\hat Z_t^2   -   Z_t^2) p_tp_t^{\transp}  \bigg\|_\infty  \\
			& \leq M''_n  \sqrt{n  (\hat \theta_n - \theta)^{\transp}  \left( \sum_{t=1}^n \phi(x_t,p_t)\phi(x_t,p_t)^{\transp} \right) (\hat \theta_n - \theta)} \nonumber \\
			& = M''_n \sqrt{ n (\hat \theta_n - \theta)^{\transp}  \left( V_n - \lambda I\right) (\hat \theta_n - \theta)} \nonumber \\
			& \leq M''_n \sqrt{ n (\hat \theta_n - \theta)^{\transp}  V_n  (\hat \theta_n - \theta)} \nonumber \\
			& = M''_n \, \big\| V_n^{1/2}\big(\theta - \hat \theta_n\big)\big\| \sqrt{n} \,,
	\end{align*}
where $V_n = \lambda I + \sum_{t=1}^n \phi(x_t,p_t)\phi(x_t,p_t)^{\transp} $
was used for the last steps.

From Lemma~\ref{lm:conc-theta} and the bound~\eqref{eq:Bn}, we finally obtain that with probability at least $1-\delta/3$,
    \begin{align}
		 \bigg\| \sum_{t=1}^n (\hat Z_t^2   -   Z_t^2) p_tp_t^{\transp}  \bigg\|_\infty & \leq M''_n \, B_n(\delta/3) \, \sqrt{n} \\
                                                                                        & = (C+2 M_n) \, B_n(\delta/3) \, \sqrt{n} \,.
                                                                                        \label{eq:first}
	\end{align}

	\medskip
	\emph{Upper bound on the second term in~\eqref{eq:twoterms}.}
Recall that $p_t$ is $\cF_{t-1}$ measurable and that in Model~1, we defined $Z_t = Y_{t,p_t} - \phi(x_t,p_t)^{\transp}   \theta = p_t^{\transp}  \epsilon_t$, which is a scalar value. These
two observations yield
	\begin{align}
& \E\big[Z_t^2 p_tp_t^{\transp}  \,\big|\, \cF_{t-1}\big] =
		\E\big[p_t Z_t^2 p_t^{\transp}  \,\big|\, \cF_{t-1}\big] \nonumber \\
& \qquad = \E\big[p_tp_t^{\transp}  \epsilon_t\epsilon_t^{\transp}  p_t p_t^{\transp} \,\big|\, \cF_{t-1}\big] \nonumber \\
& \qquad = p_tp_t^{\transp}  \, \E\big[\epsilon_t \epsilon_t^{\transp} \,\big|\, \cF_{t-1}\big] \, p_tp_t^{\transp}
= p_tp_t^{\transp}  \Gamma p_t p_t^{\transp} \,.
\label{eq:Gamma}
	\end{align}
We wish to apply the Hoeffding--Azuma inequality to each component of $Z_t^2 p_tp_t^{\transp}$,
however, we need some boundedness to do so. Therefore, we consider instead $Z_t^2 \ind{|Z_t| \leq M_n}$.
The indicated inequality, together with a union bound, entails that with probability at least $1-\delta/3$,
\begin{align}
& \Bigg\| \sum_{t=1}^n Z_t^2 \ind{|Z_t| \leq M_n} p_tp_t^{\transp}  \nonumber \\
& \qquad -
\sum_{t=1}^n \E\Big[Z_t^2  \ind{|Z_t| \leq M_n} p_t p_t^{\transp}  \,\Big|\, \cF_{t-1}\Big]
\Bigg\|_\infty \nonumber \\
\leq & M^2_n \sqrt{2n\log(3K^2/\delta)}\,.
\label{eq:hoeff-azuma}
\end{align}
Over $\mathcal{E}_n(\delta)$, using~\eqref{eq:Gamma} and applying a triangular inequality, we obtain
\begin{align}
& \bigg\| \sum_{t=1}^n Z_t^2 p_tp_t^{\transp} -p_tp_t^{\transp} \Gamma p_tp_t^{\transp} \bigg\|_\infty  \nonumber \\
& =   \bigg\| \sum_{t=1}^n Z_t^2 \ind{|Z_t| \leq M_n}  p_tp_t^{\transp} -\sum_{t=1}^n
\E\big[Z_t^2 p_tp_t^{\transp}  \,\big|\, \cF_{t-1}\big] \bigg\|_\infty  \nonumber  \\
&\leq  \bigg\| \sum_{t=1}^n Z_t^2 \ind{|Z_t| \leq M_n}  p_tp_t^{\transp}   \nonumber\\
&\qquad \qquad \qquad -\sum_{t=1}^n
\E\big[Z_t^2 p_tp_t^{\transp} \ind{|Z_t| \leq M_n}  \,\big|\, \cF_{t-1}\big]\bigg\|_\infty  \nonumber \\
&\qquad +\sum_{t=1}^n \bigg\| \E\big[Z_t^2 p_tp_t^{\transp} \ind{|Z_t| > M_n}  \,\big|\, \cF_{t-1}\big] \bigg\|_\infty \,.
\label{eq:inq_triangular}
\end{align}
We just need to bound the last term of the inequality  above to conclude this part.
Using that $x^2 \leq \exp(x)$ for $x \geq 0$, we get
\begin{multline*}
\E\Big[Z_t^2 \ind{|Z_t| > M_n} \,\Big|\, \cF_{t-1}\Big] \\
\leq \E\Big[\exp\bigl( |Z_t| \bigr) \ind{|Z_t| > M_n} \,\Big|\, \cF_{t-1}\Big]
\,.
\end{multline*}
Applying a conditional Cauchy-Schwarz inequality yields
\begin{align*}
 &\E\Big[\exp\bigl( |Z_t| \bigr) \ind{|Z_t| > M_n} \,\Big|\, \cF_{t-1}\Big]  \\
& \leq \sqrt{ \E\big[\exp\bigl(2 |Z_t| \bigr) \,\big|\, \cF_{t-1}\big] \,\, \E\big[\ind{|Z_t| > M_n} \,\big|\, \cF_{t-1}\big] }\,.
\end{align*}
 Now, thanks to the sub-Gaussian property of $\epsilon_t$
used with $\nu = 2 p_t$ and $\nu = - 2 p_t$, we have
\begin{align*}
\lefteqn{\E\big[\exp\bigl(2 |Z_t| \bigr)} \\
& \leq \E\big[\exp(2Z_t) \,\big|\, \cF_{t-1}\big] + \E\big[\exp(-2Z_t) \,\big|\, \cF_{t-1}\big] \\
& \leq 2\exp(2\rho)\,.
\end{align*}
The bound~\eqref{eq:6n} and its symmetric version indicate that
\[
\mathbb{P} \big( |Z_t| \geq M_n \,\big|\, \mathcal{F}_{t-1} \big) \leq \frac{\delta}{3n}\,.
\]
We therefore proved
\begin{align*}
\E\Big[\exp\bigl( |Z_t| \bigr) \ind{|Z_t| > M_n} \,\Big|\, \cF_{t-1}\Big] \leq \sqrt{ 2 \exp(2\rho) \, \frac{ \delta}{3n}}
\,.
\end{align*}
Thus, we have $\E\big[Z_t^2 \ind{|Z_t| > M_n} \,\big|\, \cF_{t-1}\big] \leq~2\sqrt{ \exp(2\rho) \delta/(6n)}$ and as all components of the $p_t p_t^{\transp} $ are in $[0,1]$,
\begin{equation}
\Bigl\Arrowvert
\E\big[Z_t^2 \ind{|Z_t| > M_n} p_t p_t^{\transp} \,\big|\, \cF_{t-1}\big]
\Bigr\Arrowvert_\infty \leq 2\sqrt{ \exp(2\rho)\frac{ \delta}{6n}}.
\label{eq:z_t}
\end{equation}
Finally , combining~\eqref{eq:inq_triangular} with~\eqref{eq:hoeff-azuma} and~\eqref{eq:z_t}, we get with probability $1-\delta/3$
\begin{align*}
& \bigg\| \sum_{t=1}^n Z_t^2 p_tp_t^{\transp} -p_tp_t^{\transp} \Gamma p_tp_t^{\transp} \bigg\|_\infty \\
& \leq M^2_n \sqrt{2n\log(3K^2/\delta)} + 2n\sqrt{ \exp(2\rho) \delta/(6n)} = M'_n \sqrt{n}\,,
\end{align*}
where $M'_n$ is defined in the statement of the lemma.

	\emph{Combining the two upper bounds into~\eqref{eq:twoterms}}.
	Combining the above upper bound with~\eqref{eq:twoterms} and~\eqref{eq:first}, we proved that with probability $1-\delta$,
	\begin{align*}
		\bigg\|  \sum_{t=1}^n & p_t p_t^{\transp}  \Big( \hat \Gamma_n  -\Gamma\Big) p_tp_t^{\transp}  \bigg\|_{\infty} \\
		& \leq M'_n \sqrt{n} + M''_n B_n(\delta/3)\sqrt{n}\,,
	\end{align*}
which concludes the proof.
\end{proof}

\subsection*{Conclusion of the proof of Lemma~\ref{lem:gamma}}

Remember from Section~\ref{sec:pij} that all vectors $p^{(i,j)}$
are played at least $n_0$ times in the $n$
exploration rounds.

\emph{Proof of Lemma~\ref{lem:gamma}.}~~Applying Lemma~\ref{lem:gammahat1} together with
\begin{multline}
	p_tp_t^{\transp}  \big(\hat \Gamma_n -\Gamma\big) p_tp_t^{\transp}  = p_t\Tr\Big(p_t^{\transp}  \big(\hat \Gamma_n -\Gamma\big) p_t\Big) p_t^{\transp}  \\
	= \Tr\Big( \big(\hat \Gamma_n -\Gamma\big) p_tp_t^{\transp} \Big) p_tp_t^{\transp}
\end{multline}
we have, with probability at least $1-\delta$, that for all pairs of coordinates $(i,j) \in E$,
	\begin{equation}
	\label{eq:ptij}
		\left|  \sum_{t=1}^n \Tr\Big( \big(\hat \Gamma_n -\Gamma\big) p_tp_t^{\transp} \Big) \big[p_tp_t^{\transp} \big]_{i,j} \right| \leq \kappa_n \sqrt{n} \,.
	\end{equation}

Remember that in the set $E$ considered
in Section~\ref{sec:pij}, we only have pairs $(i,j)$ with $i \leq j$.
However, for symmetry reasons, it will be convenient to also consider the vectors
$p^{(i,j)}$ with $i > j$, where the latter vectors
are defined in an obvious way. We note that for all $1 \leq i,j \leq K$,
\begin{equation} \label{eq:sym-p}
p^{(i,j)}{p^{(i,j)}}^{\transp} = p^{(j,i)}{p^{(j,i)}}^{\transp}\,.
\end{equation}

Now, our aim is to control
\begin{equation}
\label{eq:aimLM2}
\Big| q^{\transp}   \big(\hat \Gamma_n - \Gamma \big) q \Big| = \bigg| \Tr\Big(\big( \hat \Gamma_n - \Gamma \big) qq^{\transp} \Big) \bigg|
\end{equation}
uniformly over $q \in \cP$.
The proof consists of two steps:
establishing such a control for the special cases where $q$ is one of the $p^{(i,j)}$
and then, extending the control to arbitrary vectors $q \in \cP$, based on a decomposition
of $q q^{\transp}$ as a weighted sum of $ p^{(i,j)}{p^{(i,j)}}^{\transp}$ vectors.
\medskip

\emph{Part 1: The case of the $p^{(i,j)}$ vectors.}
Consider first the off-diagonal elements $1\leq i<j\leq K$.
Note that since $p_t$ is of the form $p^{(i',j')}$ for all $1 \leq t \leq n$, we have
\begin{equation}
		\big[p_tp_t^{\transp} \big]_{i,j} = \left\{
			\begin{array}{ll}
				1/4 & \text{if } p_t = p^{(i,j)}, \\
				0 & \text{otherwise.}
			\end{array}
			\right.
				 \label{eq:p_ij}
\end{equation}
Using that $p_t = p^{(i,j)}$ at least for $n_0$ rounds, Inequality~\eqref{eq:ptij} entails
\[
	 \frac{n_0}{4} \left|  \Tr\Big( \big(\hat \Gamma_n -\Gamma\big) p^{(i,j)}{p^{(i,j)}}^{\transp} \Big) \right| \leq \kappa_n \sqrt{n}\,,
\]
or put differently,
\begin{equation}
	 \left|\Tr\Big( \big(\hat \Gamma_n -\Gamma\big) p^{(i,j)}{p^{(i,j)}}^{\transp}  \Big) \right| \leq \frac{4 \kappa_n \sqrt{n}}{n_0} \,.
	 \label{eq:gammaij}
\end{equation}

Now, let us consider the diagonal elements. Let $1\leq i\leq K$. We have
\begin{equation}
	\big[p_tp_t^{\transp} \big]_{i,i} = \left\{
		\begin{array}{ll}
			 1 & \text{if } p_t = p^{(i,i)}, \\
			 1/4 & \text{if } p_t = p^{(i,j)} \text{ for some } j > i, \\
			 1/4 & \text{if } p_t = p^{(k,i)} \text{ for some } k < i, \\
			 0 & \text{otherwise,}
		\end{array} \right.
	 \label{eq:p_ii}
\end{equation}
where we recall that the $p_t$ are necessarily of the form $p^{(k,\ell)}$ with $k \leq \ell$.
Therefore, Inequality~\eqref{eq:ptij} yields
\begin{multline*}
	n_0 \Biggl|\Tr\bigg(\big(\hat \Gamma_n - \Gamma\big) \Big(p^{(i,i)}{p^{(i,i)}}^{\transp}  +  \frac{1}{4} \sum_{j > i} p^{(i,j)}{p^{(i,j)}}^{\transp} \\
+ \frac{1}{4} \sum_{k < i} p^{(k,i)}{p^{(k,i)}}^{\transp} \Big)\bigg) \Biggr|
	 \leq \kappa_n \sqrt{n}\,,
\end{multline*}
which we rewrite by symmetry---see~\eqref{eq:sym-p}---as
\begin{multline}
\Biggl|\Tr\bigg(\big(\hat \Gamma_n - \Gamma\big) \Big(p^{(i,i)}{p^{(i,i)}}^{\transp}  +  \frac{1}{4} \sum_{j \ne i} p^{(i,j)}{p^{(i,j)}}^{\transp}
\Big)\bigg) \Biggr| \\ \leq \frac{\kappa_n \sqrt{n}}{n_0} \,.
	 \label{eq:gammaii}
\end{multline}

\emph{Part 2-1: Decomposing arbitrary vectors $q \in \cP$.}
Now, let $q \in \cP$. We show below by means of elementary calculations that
\begin{equation}
	qq^{\transp}  = \sum_{i=1}^K \sum_{j=1}^K u(i,j) \  p^{(i,j)}{p^{(i,j)}}^{\transp}
	\label{eq:qq-transp}
\end{equation}
with $u(i,j) = 2 q_iq_j$ if $i\neq j$ and $u(i,i)=2q_i^2-q_i$.

Indeed, by identification and by imposing $u(i,j) = u(j,i)$ for all pairs $i,j$,
the equalities~\eqref{eq:p_ij}
and the symmetry property~\eqref{eq:sym-p} entail, for $k \ne k'$:
\begin{align*}
q_k q_{k'}=\big[ qq^{\transp} \big]_{k,k'} &= \sum_{i=1}^K \sum_{j=1}^K u(i,j) \,\,  \big[ p^{(i,j)}{p^{(i,j)}}^{\transp} \big]_{k,k'} \\
	& = \frac{u(k,k')}{4} + \frac{u(k',k)}{4} = \frac{u(k,k')}{2}\,,
\end{align*}
which can be rephrased as $u(k,k') = u(k',k) = 2 q_k q_{k'}$.
Now, let us calculate the diagonal elements, by identification and by the equalities~\eqref{eq:p_ii} as well as by
the symmetry property~\eqref{eq:sym-p}:
\begin{align*}
q_k^2&=\big[ qq^{\transp} \big]_{k,k} = \sum_{i=1}^K \sum_{j=1}^K u(i,j) \  \bigl[ p^{(i,j)}{p^{(i,j)}}^{\transp} \bigr]_{k,k} \\
	&=  u(k,k) + \sum_{i\neq k}  \frac{u(i,k)}{4} + \sum_{j\neq k}  \frac{u(k,j)}{4} \\
	&=  u(k,k) + \frac{1}{2} \sum_{i\neq k}  u(i,k) = u(k,k) + \sum_{i\neq k} q_k q_i \\
	&= u(k,k) + \sum_{i=1}^K q_k q_i - q_k^2 = u(k,k) + q_k - q_k^2\,,
\end{align*}
which leads to $u(k,k)=2q_k^2-q_k$.

We introduce the notation
\[
P^{(i,j)} = p^{(i,j)}{p^{(i,j)}}^{\transp}
\]
and in light of~\eqref{eq:gammaij} and~\eqref{eq:gammaii},
we rewrite~\eqref{eq:qq-transp} as
\begin{align*}
qq^{\transp}  & = \sum_{i=1}^K u(i,i) \left( P^{(i,i)}+ \frac{1}{4} \sum_{j \ne i} P^{(i,j)} \right) \\
& \quad + \sum_{i=1}^K  \sum_{j \neq i} \bigg( u(i,j) - \frac{u(i,i)}{4} \bigg) P^{(i,j)}\,.
\end{align*}
\medskip

\emph{Part 2-2: Controlling arbitrary vectors $q \in \cP$.}
Therefore, substituting this decomposition of $qq^{\transp}$ into the aim~\eqref{eq:aimLM2},
and using the linearity of the trace as well as the triangle inequality for absolute values,
we obtain
\begin{align*}
	& \Big| q^{\transp}   \big(\hat \Gamma_n - \Gamma \big) q \Big|
		 = \Big| \Tr\Big(\big( \hat \Gamma_n - \Gamma \big) qq^{\transp} \Big) \Big| \\
		& \leq \sum_{i=1}^K \bigl| u(i,i) \bigr| \, \Bigg| \Tr\biggl(\big(\hat \Gamma_n - \Gamma\big) \Big(P^{(i,i)} + \frac{1}{4} \sum_{j \neq i} P^{(i,j)} \Big)\biggr) \Bigg| \\
        & \quad + \sum_{i=1}^K  \sum_{j\neq i} \bigg| u(i,j) - \frac{u(i,i)}{4}\bigg| \, \bigg| \Tr\Big( \big( \hat \Gamma_n -\Gamma\big) P^{(i,j)} \Big) \bigg| \\
\end{align*}
We then substitute the upper bounds~\eqref{eq:gammaij} and~\eqref{eq:gammaii}
and get
\begin{multline*}
	\Big| q^{\transp}   \big(\hat \Gamma_n - \Gamma\big) q \Big|  \\
		\leq \frac{\kappa_n \sqrt{n}}{n_0} \Bigg( \sum_{i=1}^K \bigl| u(i,i) \bigr|
+ 4 \sum_{i=1}^K  \sum_{j\neq i} \bigg| u(i,j) - \frac{u(i,i)}{4}\bigg| \Bigg).
\end{multline*}
By the triangle inequality, by the values $2 q_i q_j$ of the coefficients $u(i,j)$ when $i \ne j$ and
by using $|u(i,i)| \leq q_i$,
\begin{align*}
\lefteqn{\sum_{i=1}^K \bigl| u(i,i) \bigr|
+ 4 \sum_{i=1}^K  \sum_{j\neq i} \bigg| u(i,j) - \frac{u(i,i)}{4}\bigg|} \\
& \leq K \sum_{i=1}^K \bigl| u(i,i) \bigr|
+ 4 \sum_{i=1}^K  \sum_{j\neq i} \big| u(i,j) \big| \\
& \leq K \sum_{i=1}^K q_i + 8 \sum_{i=1}^K  \sum_{j\neq i} q_i q_j \\
& = K + 8 \sum_{i=1}^K  q_i (1-q_i) \leq K+8\,.
\end{align*}
Putting all elements together, we proved
\[
\sup_{q \in \cP} \left| q^{\transp} \big(\hat\Gamma_n - \Gamma\big) q \right|
\leq \frac{\kappa_n \sqrt{n}}{n_0} (K+8)\,,
\]
which concludes the proof of Lemma~\ref{lem:gamma}.
\hfill \qed

\newpage
\section{Proof of Lemma~\ref{lm:main}}
\label{sec:prooflmmain}

We recall that this lemma is a straightforward adaptation/generalization of Lemma~19.1 of the monograph by~\citet{lattimore2018bandit};
see also a similar result in Lemma~3 by~\citet{chu2011contextual}.

We consider the worst case when all summations would start at $n+1 = 2$.

By definition, the quantity $\oB$
upper bounds all the $B_{t-1}(\delta t^{-2})$. It therefore suffices to upper bound
\begin{align*}
& \sum_{t=2}^T \min \Bigl\{ L, \,\, 2C \oB \, \bnorm{V_{t-1}^{-1/2} \phi(x_t,p_t)} \Bigr\} \\
& \leq \sqrt{T} \sqrt{\sum_{t=2}^T \min \Bigl\{ L^2, \,\, \bigl( 2C \oB \bigr)^2 \, \bnorm{V_{t-1}^{-1/2} \phi(x_t,p_t)}^2 \Bigr\}} \\
& = \sqrt{T} \sqrt{\sum_{t=2}^T \min \Biggl\{ L^2, \,\, \bigl( 2C \oB \bigr)^2 \, \biggl( \frac{\det(V_t)}{\det(V_{t-1})} - 1 \biggr) \Biggr\}}
\end{align*}
where we applied first the Cauchy-Schwarz inequality
and used second the equality
\begin{multline*}
1 + \bnorm{V_{t-1}^{-1/2} \phi(x_t,p_t)}^2 \\
= 1 + \phi(x_t,p_t)^{\transp} V_{t-1}^{-1} \phi(x_t,p_t) = \frac{\det(V_t)}{\det(V_{t-1})}\,,
\end{multline*}
that follows from a standard result in online matrix theory, namely, Lemma~\ref{lm:onlinematrix} below.

Now, we get a telescoping sum with the logarithm function by using the inequality
\begin{equation}
\label{eq:ulnu}
\forall b >0, \quad \forall u > 0, \qquad \min\{b,u\} \leq b \, \frac{\ln(1+u)}{\ln(1+b)}\,,
\end{equation}
which is proved below.
Namely, we further bound the sum above by
\begin{align*}
& \sum_{t=2}^T \min \Biggl\{ L^2, \,\, \bigl( 2C \oB \bigr)^2 \, \biggl( \frac{\det(V_t)}{\det(V_{t-1})} - 1 \biggr) \Biggr\} \\
& \leq \bigl( 2C \oB \bigr)^2 \sum_{t=2}^T \min \Biggl\{ \frac{L^2}{\bigl( 2C \oB \bigr)^2}, \,\, \frac{\det(V_t)}{\det(V_{t-1})} - 1 \Biggr\} \\
& \leq \bigl( 2C \oB \bigr)^2 \sum_{t=2}^T \frac{L^2/\bigl( 2C \oB \bigr)^2}{\ln \Bigl( 1 + L^2/\bigl( 2C \oB \bigr)^2 \Bigr)}
\ln \biggl( \frac{\det(V_t)}{\det(V_{t-1})} \biggr) \\
& = \frac{L^2}{\ln \Bigl( 1 + L^2/\bigl( 2C \oB \bigr)^2 \Bigr)} \ln \biggl( \frac{\det(V_T)}{\det(V_{2})} \biggr)\\
& \leq \frac{L^2}{\ln \Bigl( 1 + L^2/\bigl( 2C \oB \bigr)^2 \Bigr)} \, d \ln \frac{\lambda+T}{\lambda}
\end{align*}
where we used~\eqref{eq:boundedness} and one of its consequences to get the last inequality.

Finally, we use $1/ \ln(1+u) \leq 1/u + 1/2$ for all $u \geq 0$ to get a more
readable constant:
\[
\frac{L^2}{\ln \Bigl( 1 + L^2/\bigl( 2C \oB \bigr)^2 \Bigr)}
\leq \bigl( 2C \oB \bigr)^2 + \frac{L^2}{2}\,.
\]
The proof is concluded by collecting all pieces. \hfill \qed

Finally, we now provide the proofs of two either straightforward or standard
results used above.

\subsection{A Standard Result in Online Matrix Theory}

The following result is extremely standard in online matrix theory (see,
among many others, Lemma 11.11 in~\citealp{cesa2006prediction}
or the proof of Lemma~19.1 in the monograph by~\citealp{lattimore2018bandit}).

\begin{lemma}
\label{lm:onlinematrix}
Let $M$ a $d \times d$ full-rank matrix, let $u,\,v \in \R^d$ be two arbitrary vectors. Then
\[
1+ v^{\transp} M^{-1} u = \frac{\det\bigl(M+u v^{\transp})}{\det(M)} \,.
\]
\end{lemma}

The proof first considers the case $M = \Id$.
We are then left with showing that $\det\bigl(\Id +u v^{\transp} \bigr) = 1 + v^{\transp} u$, which
follows from taking the determinant of every term of the equality
\begin{align*}
& \left[ \begin{array}{cc}
\Id & 0 \\
v^{\transp} & 1
\end{array}
\right]
\left[ \begin{array}{cc}
\Id+u v^{\transp} & u \\
0 & 1
\end{array}
\right]
\left[ \begin{array}{cc}
\Id & 0 \\
-v^{\transp} & 1
\end{array}
\right] \\
= &
\left[ \begin{array}{cc}
\Id & u \\
0 & 1+v^{\transp} u
\end{array}
\right] \,.
\end{align*}

Now, we can reduce the case of a general $M$ to this simpler case by noting that
\begin{align*}
\det\bigl(M+u v^{\transp})
& = \det(M) \,\, \det\Bigl(\Id+\bigl(M^{-1}u\bigr) v^{\transp} \Bigr) \\
& = \det(M) \,\, \bigl(1+v^{\transp}M^{-1}u\bigr)\,.
\end{align*}

\subsection{Proof of Inequality~\eqref{eq:ulnu}}

This inequality is used in Lemma~19.1 of the monograph by~\citet{lattimore2018bandit}, in the special case $b=1$.
The extension to $b > 0$ is straightforward.

We fix $b > 0$.
We want to prove that
\begin{equation}
\label{eq:ulnu2}
\forall u > 0, \qquad \min\{b,u\} \leq b \, \frac{\ln(1+u)}{\ln(1+b)}\,.
\end{equation}
We first note that
\[
\min\{b,u\} = b \, \frac{\ln(1+u)}{\ln(1+b)} \quad \mbox{for} \ u=b
\]
and that $\min\{b,u\} = b$ for $u \geq b$, with the right-hand side of~\eqref{eq:ulnu2}
being an increasing function of~$u$. Therefore, it suffices to prove~\eqref{eq:ulnu2}
for $u \in [0,b]$, where $\min\{b,u\} = u$. Now,
\[
u \longmapsto b \, \frac{\ln(1+u)}{\ln(1+b)} - u
\]
is a concave and (twice) differentiable function, vanishing at $u = 0$ and
$u = b$, and is therefore non-negative on $[0,b]$. This concludes the proof.

\newpage
\section{Proof of Theorem~\ref{th:2}}
\label{sec:th2}

\begin{comm}
The key observation lies in Step~1 (and is tagged as such); the rest is standard maths.
\end{comm}

Because of the expression for the expected losses~\eqref{eq:ellmod2}
and the consequence~\eqref{eq:csq-att} of attainability,
the regret can be rewritten as
\[
R_{T} = \sum_{t=1}^T \ell_{t,p_t} = \sum_{t=1}^T \bigl( \phi(x_t,p_t)^{\transp} \theta - c_t \bigr)^2\,.
\]
We first successively prove (\emph{Step~1}) that for $t \geq 2$, if
the bound of Lemma~\ref{lm:conc-theta} holds, namely,
\begin{equation}
\label{eq:goodhth-th2}
\Bnorm{V_{t-1}^{1/2} \bigl( \theta - \hth_{t-1} \bigr)} \leq B_{t-1}(\delta t^{-2})\,,
\end{equation}
then
\begin{align}
\label{eq:proof-th2-1}
 &\ell_{t,p_t} \leq 2 \beta_{t,p_t} + 2 \tl_{t,p_t}\,, \\
 \label{eq:proof-th2-2}
 &\tl_{t,p_t} \leq \beta_{t,p_t} + \tl_{t,p^\star_t} - \beta_{t,p^\star_t}\,, \\
 \label{eq:proof-th2-3}
&\tl_{t,p^\star_t} \leq \beta_{t,p^\star_t}\,.
\end{align}
These inequalities collectively entail the bound $\ell_{t,p_t}\leq~4 \beta_{t,p_t}$.
Of course, because of the boundedness assumptions~\eqref{eq:boundedness},
we also have $\ell_{t,p_t} \leq C^2$.
It then suffices to bound the sum (\emph{Step~2}) of
the $\ell_{t,p_t}$ by the sum of the $\min \bigl\{ C^2, 4 \beta_{t,p_t} \bigr\}$
and control for the probability of~\eqref{eq:goodhth-th2}.

\emph{Step~1: Proof of~\eqref{eq:proof-th2-1}--\eqref{eq:proof-th2-3}.}
Inequality~\eqref{eq:proof-th2-2} holds by definition of the algorithm.
For~\eqref{eq:proof-th2-3} and~\eqref{eq:proof-th2-1}, we re-use the inequality~\eqref{eq:CS-phitheta}
proved earlier: for all $p \in \cP$,
\begin{align}
\nonumber
& \Bigl( \phi(x_t,p)^{\transp} \bigl( \theta - \hth_{t-1} \bigr) \Bigr)^2 \\
& \leq \Bnorm{V_{t-1}^{1/2} \bigl( \theta - \hth_{t-1} \bigr)}^2 \, \bnorm{V_{t-1}^{-1/2} \phi(x_t,p)}^2 \\
\label{eq:phip-th2}
& \leq B_{t-1}(\delta t^{-2})^2 \, \bnorm{V_{t-1}^{-1/2} \phi(x_t,p)}^2 \defeq \beta_{t,p}\,,
\end{align}
where we used the bound~\eqref{eq:goodhth-th2} for the last inequality.
This inequality directly yields~\eqref{eq:proof-th2-3}
by taking $p = p^\star_t$.

Now comes the specific improvement and our key observation:
using that $(u+v)^2 \leq 2u^2 + 2v^2$, we have
\begin{align*}
\ell_{t,p_t}
& = \Bigl( \phi(x_t,p_t)^{\transp} \theta - \phi(x_t,p_t)^{\transp} \hth_{t-1} \\
& \qquad \qquad + \phi(x_t,p_t)^{\transp} \hth_{t-1} - c_t \Bigr)^2 \\
& \leq 2 \bigl( \phi(x_t,p_t)^{\transp} \theta - \phi(x_t,p_t)^{\transp} \hth_{t-1} \bigr)^2 \\
& \qquad + 2 \underbrace{\bigl( \phi(x_t,p_t)^{\transp} \hth_{t-1} - c_t \bigr)^2}_{= \tl_{t,p_t}}\,,
\end{align*}
which yields~\eqref{eq:proof-th2-1} via~\eqref{eq:phip-th2} used with $p = p_t$.

\emph{Step~2: Summing the bounds.}
First, the bound~\eqref{eq:goodhth-th2} holds, by Lemma~\ref{lm:conc-theta},
with probability at least $1-\delta t^{-2}$ for a given $t \geq 2$. By a union bound,
it holds for all $t \geq 2$ with probability at least $1-\delta$. By bounding $\ell_{t,p_t}$
by~$C^2$ and the $B_{t-1}(\delta t^{-2})$ by $\oB$, we therefore get,
from Step~1, that with probability at least $1-\delta$,
\[
\oR_T \leq C^2 + \sum_{t=2}^T \min \Bigl\{ C^2, \,\, 4 \oB^2 \bnorm{V_{t-1}^{-1/2} \phi(x_t,p)}^2 \Bigr\} \,.
\]
Now, as in the proof of Lemma~\ref{lm:main} above (Appendix~\ref{sec:prooflmmain}),
\begin{align*}
& \sum_{t=2}^T \min \Bigl\{ C^2, \,\, 4 \oB^2 \bnorm{V_{t-1}^{-1/2} \phi(x_t,p)}^2 \Bigr\} \\
& = \sum_{t=2}^T \min \Biggl\{ C^2, \,\, 4 \oB^2 \biggl( \frac{\det(V_T)}{\det(V_{1})} -1 \biggr) \Biggr\} \\
& \leq 4 \oB^2 \sum_{t=2}^T \frac{C^2/ \bigl( 4 \oB^2 \bigr)}{\ln \Bigl( 1 + C^2/ \bigl( 4 \oB^2 \bigr) \Bigr)}
\ln \biggl( \frac{\det(V_t)}{\det(V_{t-1})} \biggr) \\
& = \frac{C^2}{\ln \Bigl( 1 + C^2/ \bigl( 4 \oB^2 \bigr) \Bigr)} \ln \biggl( \frac{\det(V_T)}{\det(V_{1})} \biggr) \\
& \leq \biggl( 4 \oB^2 + \frac{C^2}{2} \biggr) \, d \ln \frac{\lambda+T}{\lambda}\,.
\end{align*}
This concludes the proof.

\newpage
\section{Numerical expression of the covariance matrix $\Gamma$ built on data}
\label{app:GammaEst}

The covariance matrix $\Gamma$ was built based on historical data as indicated in Section~\ref{subsec:simulator}.
Namely, we considered the time series of residuals associated with our estimation of the consumption.
The diagonal coefficients $\Gamma_{j,j}$ were given by the empirical variance of the residuals associated with tariff $j$, 
while non-diagonal coefficients $\Gamma_{j,j'}$ were given by the empirical covariance between residuals of tariffs $j$ and $j'$ 
at times $t$ and $t \pm 48$. (A more realistic model might consider a noise which depends on the half-hour of the day).

\textbf{Numerical expression obtained.}
More precisely, the variance terms $\Gamma_{1,1}$, $\Gamma_{2,2}$, and $\Gamma_{3,3}$ were computed with 
respectively $788$, $15 \, 072$ and $1 \, 660$ observations, while the non-diagonal coefficients were based on fewer observations: $1\,318$ for $\Gamma_{2,3}$ 
and $620$ for $\Gamma_{1,2}$, but only $96$ for $\Gamma_{1,3}$.
The resulting matrix $\Gamma$ is
\begin{equation*}
\Gamma=\sigma^2 \begin{pmatrix}
1.11 & 0.46 & 0.04 \\
0.46 & 1.00  & 0.56 \\
0.04 & 0.56 & 2.07
\end{pmatrix} \quad \mathrm{with} \quad \sigma=0.02.
\end{equation*}
To get an idea of the orders of magnitude at stake, we indicate that 
in the data set considered, the mean consumption remained between $0.08$ and $0.21$ kWh per half-hour
and that its empirical average equals $0.46$.

\textbf{Off-diagonal coefficients are non-zero.}
We may test, for each $j \neq j'$, the null hypothesis $\Gamma_{j,j'}=0$ using the Pearson correlation test; 
we obtain low p--values (smaller than something of the order of $10^{-13}$), which shows that $\Gamma$ is significantly
different from a diagonal matrix. We may conduct a similar study to show that it is not proportional to the 
all-ones matrix, nor to any matrix with a special form.

\end{document}